\pdfoutput=1
\documentclass[10pt]{scrartcl}

\usepackage[english]{babel} 
\usepackage[protrusion=true,expansion=true]{microtype} 
\usepackage{amsmath,amsfonts,amsthm,amssymb,color,epsfig,fullpage} 
\usepackage{verbatim,fancyvrb}
\usepackage{alltt}
\usepackage{tikz}
\usepackage{enumerate}
\usepackage{caption}
\usepackage{subfigure}
\usepackage[]{algorithm}
\usepackage{algpseudocode}
\usepackage{stmaryrd}
\usepackage{afterpage}
\usepackage{url}

\newtheorem{definition}{Definition}[section]
\newtheorem{theorem}[definition]{Theorem}
\newtheorem{corollary}[definition]{Corollary}
\newtheorem{proposition}[definition]{Proposition}
\newtheorem{lemma}[definition]{Lemma}
\newtheorem{example}[definition]{Example}
\newtheorem{remark}[definition]{Remark}

\def\N{{\mathbb N}}

\def\R{{\mathbb R}}

\DeclareMathOperator{\dtw}{\delta}

\newcommand{\commentout}[1]{}

\newcommand{\abs}[1]{\mathop{\left\lvert #1 \right\rvert}} 
\newcommand{\args}[1]{\mathop{\left( #1 \right)}} 

\newcommand{\norm}[1]{\mathop{\left\lVert #1 \right\rVert}}
\newcommand{\cbrace}[1]{\mathop{\left\{ #1 \right\}}}
\newcommand{\bracket}[1]{\mathop{\left[ #1 \right]}}

\newcommand{\argsS}[2]{\mathop{\left( #1 \right)#2}} 

\newcommand{\normS}[2]{\mathop{\left\lVert #1 \right\rVert#2}}

\DeclareMathOperator{\bd}{bd}         			

\renewcommand{\S}[1]{{\mathcal{#1}}}           	

\renewenvironment{cases}{%
\left\{\begin{array}{c@{\quad : \quad}l}}%
{%
\end{array}\right.}

\usepackage{framed}

\begin{document}
\title{On the Existence of a Sample Mean in Dynamic Time Warping Spaces} 

\author{Brijnesh Jain and David Schultz \\ Technische Universit\"at Berlin, Germany \\ e-mail: brijnesh.jain@gmail.com} 

\date{} 

\maketitle 

\paragraph*{Abstract.}
The concept of sample mean in dynamic time warping (DTW) spaces has been successfully applied to improve pattern recognition systems and generalize centroid-based clustering algorithms. Its existence has neither been proved nor challenged. This article presents sufficient conditions for existence of a sample mean in DTW spaces. The proposed result justifies prior work on approximate mean algorithms, sets the stage for constructing exact mean algorithms, and is a first step towards a statistical theory of DTW spaces. 

\begin{framed}
\begin{small}
\tableofcontents
\end{small}
\end{framed}

\section{Introduction}\label{sec:intro} 

Time series are time-dependent observations that vary in length and temporal dynamics (speed). Examples of time series data include acoustic signals, electroencephalogram recordings, electrocardiograms, financial indices, and internet traffic data.

Time series averaging aims at finding a typical time series that ``best'' represents a given sample of time series. First works on time series averaging started in the 1970ies with speech recognition as the prime application \cite{Rabiner1979}. Since then, research predominantly focused on devising averaging algorithms for improving pattern recognition systems and  generalizing centroid-based clustering algorithms  \cite{Abdulla2003,Gupta1996,Hautamaki2008,Kruskal1983,Niennattrakul2009,Petitjean2011,Petitjean2014,Petitjean2016,Sathianwiriyakhun2016,Soheily-Khah2016}. In contrast to averaging points in a Euclidean space, averaging time series is a non-trivial task, because the sample time series can vary in length and speed. To filter out these variations, the above cited averaging algorithms apply dynamic time warping (DTW). 

The most promising direction poses time series averaging as an optimization problem \cite{Cuturi2017,Hautamaki2008,Petitjean2011,Schultz2017,Soheily-Khah2016}: Suppose that $\S{T}$ is the set of all time series of finite length and $\S{X} = \args{x^{(1)}, \ldots, x^{(N)}}$ is a sample of $N$ time series $x^{(k)} \in \S{T}$. Then time series averaging amounts in minimizing the Fr\'echet function \cite{Frechet1948}
\begin{align}\label{eq:Frechet-Function}
F: \S{T}_m \rightarrow \R, \quad x \, \mapsto \; \frac{1}{N}\sum_{k=1}^N \delta^2\!\args{x, x^{(k)}},
\end{align}
where the solution space $\S{T}_m\subset \S{T}$ is the subset of all time series of length $m$ and $\delta(x,y)$ denotes the DTW distance between time series $x$ and $y$ \cite{Sakoe1978}. The global minimizers $z \in \S{T}_m$ of the Fr\'echet function $F(x)$ are the restricted sample means of sample $\S{X}$. The restriction refers to confining the length $m$ of the candidate solutions. 

Using Fr\'echet functions, the notion of a ``typical time series that best represents a sample'' has a precise meaning. A typical time series is any global minimizer of the Fr\'echet function. If a global minimum exists, it best represents a sample in the sense that it deviates least from all sample time series. The Fr\'echet function is motivated by the property that an arithmetic mean of real numbers minimizes the mean squared error from the sample numbers. Following Fr\'echet \cite{Frechet1948}, we can use this property to generalize the concept of sample mean to arbitrary distance spaces for which a well-defined addition is unknown and therefore an arithmetic mean can not be computed in closed form by a weighted sum.

Research on the Fr\'echet function as defined in Eq.~\eqref{eq:Frechet-Function} has the following shortcomings:
\begin{enumerate}
\itemsep0em
\item Existence of a global minimum of the Fr\'echet function has neither been proved nor challenged. Existence of an optimal solution depends on the particular choice of DTW distance and loss function. A restricted sample mean trivially exists if the DTW distance between two time series is always zero. Conversely, there are DTW spaces for which a sample mean does not always exist (cf.~Example \ref{example:non-existence}). Thus it is unclear whether existing heuristics indeed approximate a typical time series or unknowingly search for a phantom.

\item In experiments, the solution space $\S{T}_m$ is heuristically chosen. For example, if all sample time series are of length $m$ then a common choice is $\S{T}_m$. The intuition behind this choice is that the length of a restricted sample mean should ``best'' represent the lengths of the sample time series. This intuition may lead to solutions that fail to capture the characteristic properties of the sample time series as illustrated by Figure \ref{fig:ex_unrestricted_mean}. 

\item The Fr\'echet function only generalizes the sample mean. Neither weighted means nor other important measures of central location such as the sample median are captured by Eq.~\eqref{eq:Frechet-Function}. 
\end{enumerate} 

\begin{figure}[t]
\centering
\begin{tabular}{ccc}
restricted sample mean $z_3$ & restricted sample mean $z_4$ & variance \\
\includegraphics[width=0.3\textwidth]{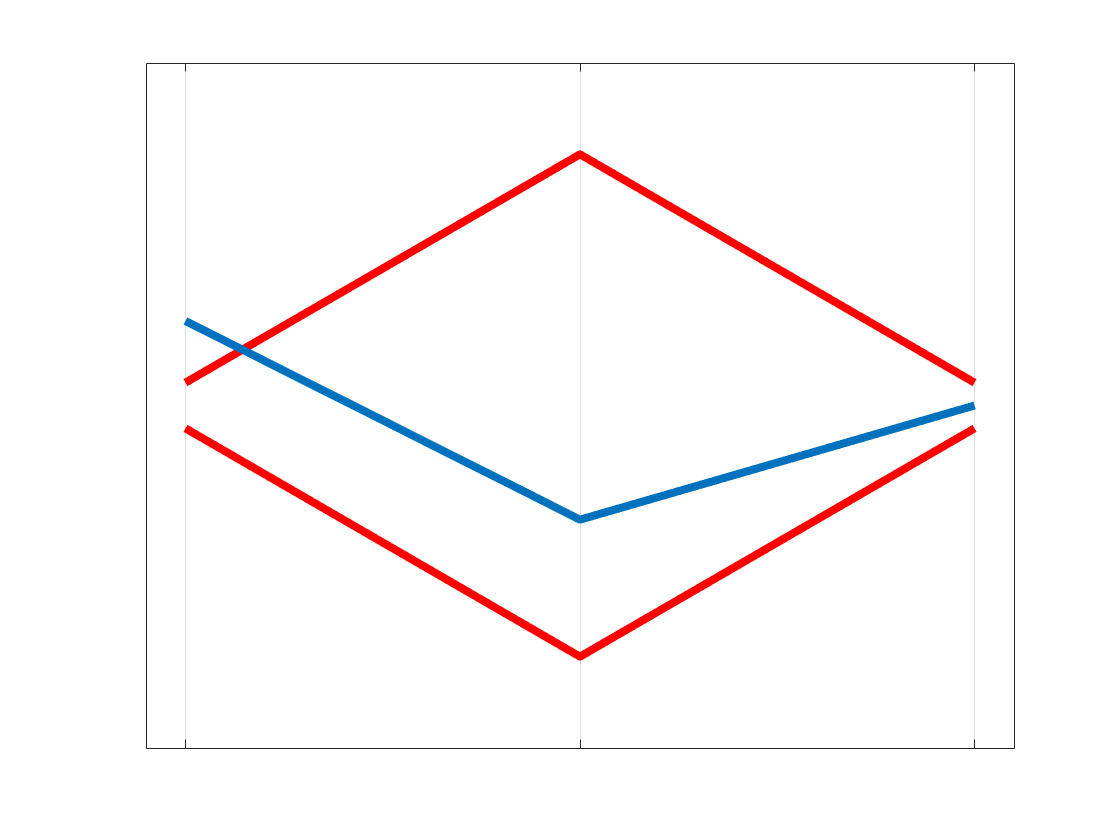}&
\includegraphics[width=0.3\textwidth]{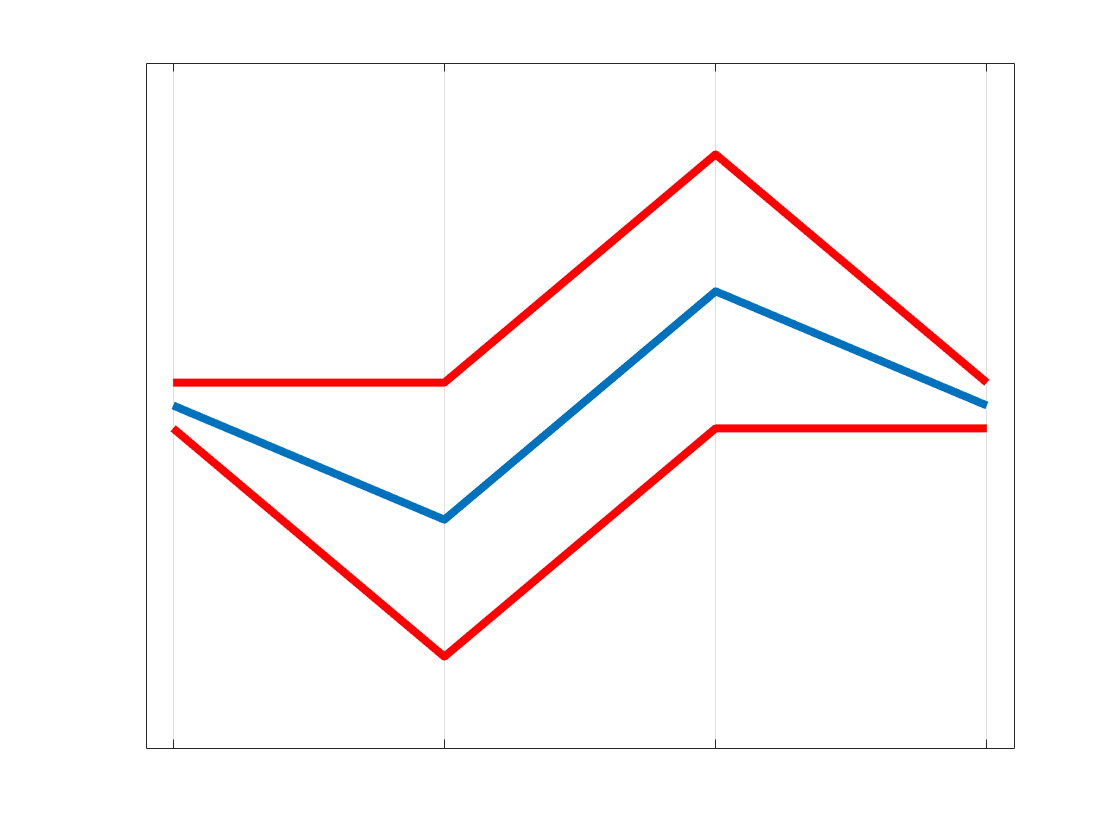}&
\includegraphics[width=0.3\textwidth]{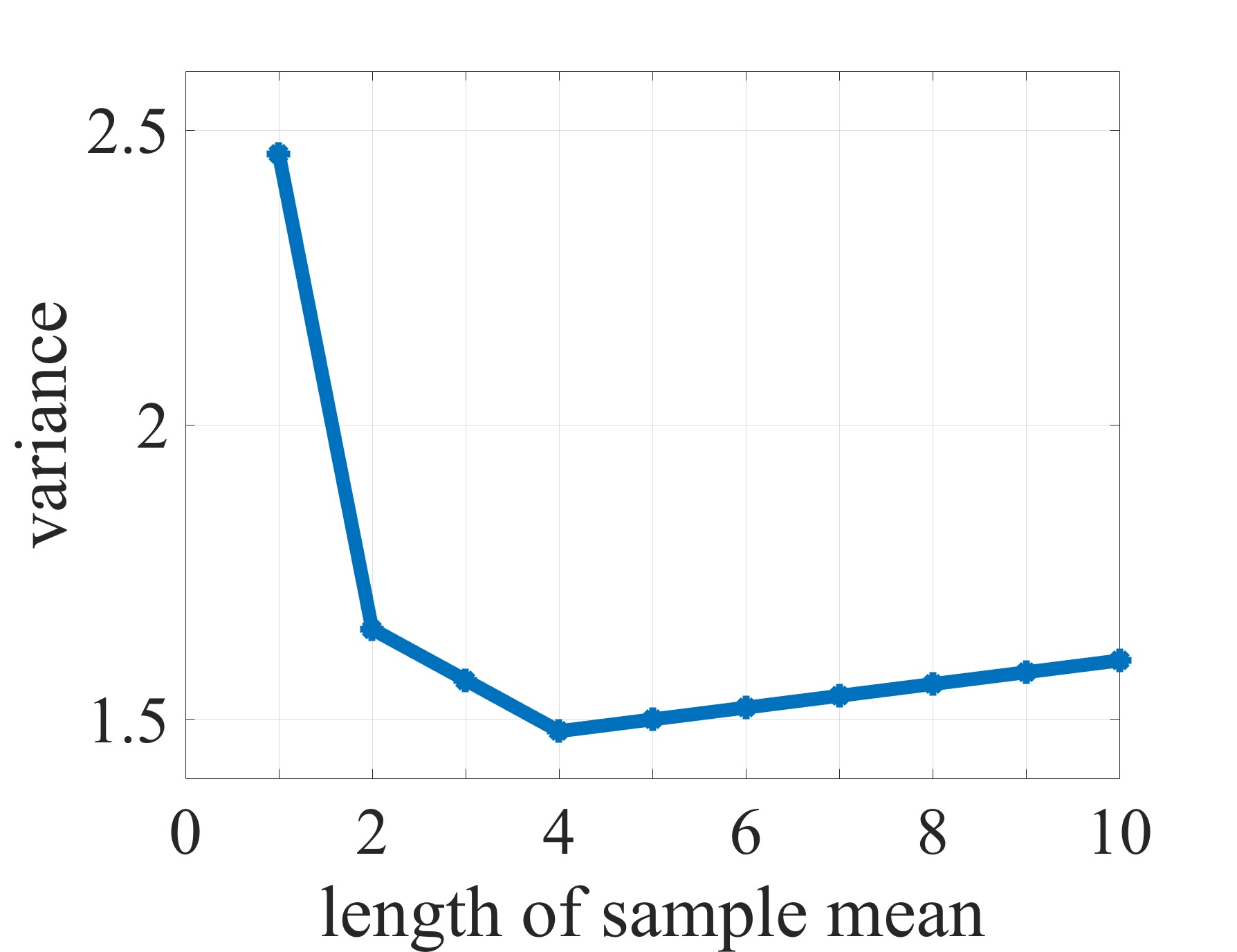}\\
(a) & (b) & (c)
\end{tabular}
\caption{Dependence of the variance (global minimum of the Fr\'echet function) on the parameter $m$ of the solutions space $\S{T}_m$. Plot (a) shows two time series (red) of length $3$ together with a restricted sample mean $z_3$ of length $3$ (blue). The restricted sample mean $z_3$ fails to properly capture the peak of one and the valley of the other sample time series. Plot (b) shows the same time series (red) warped onto the time axis of a restricted sample mean $z_4$ of length $4$ (blue). Warping increases the length of the red sample time series to length $4$. For the upper (lower) sample time series the first (third) element is copied. In contrast to $z_3$, the restricted sample mean $z_4$ captures the peak of one and the valley of the other sample time series. Plot (c) shows the variance $F(z_m)$ depending on the parameter $m \in \cbrace{1, \ldots, m}$ of the solution space $\S{T}_m$. We see that $F(z_4) \leq F(z_m)$ for all $m$ and in particular $F(z_4) < F(z_3)$. This shows that the better representational properties of $z_4$ are reflected by a lower variance.}
\label{fig:ex_unrestricted_mean}
\end{figure}

To address all three issues, we consider a more general formulation of the Fr\'echet function as given by
\begin{align*}
F: \S{U} \rightarrow \R, \quad x \mapsto \sum_{k=1}^N h_k\Big(\delta\args{x, x^{(k)}}\Big),
\end{align*}
where $\S{U} \subseteq \S{T}$ is the solution space and the $h_k:\R \rightarrow \R$ are loss functions. We recover the standard Fr\'echet function given in Eq.~\eqref{eq:Frechet-Function} by setting $h_k(u) = u^2$ for all $k$. To average the sum of squared DTW distances, we define $h_k(u) = u^2/N$, where $1/N$ is the uniform weight. To obtain a weighted version of Eq.~\eqref{eq:Frechet-Function}, we demand that $h_k(u) = w_k\cdot u^2$, where $w_k \in \R_+$ is a positive weight. The sample median is obtained by setting $h_k(u) = u$ for all $k$. Regardless of the choice of loss function, we refer to global minimizers of the general Fr\'echet function as sample means as umbrella term. 

We focus on two forms of solution sets $\S{U}$: (i) $\S{U} = \S{T}$ is the set of all time series of finite length and (ii) $\S{U} = \S{T}_m$ is the subset of time series of length $m$. We call $(i)$ the unrestricted and $(ii)$ the restricted form. Note that restrictions refer to the solution space $\S{U}$ only. Sample time series $x^{(k)}$ are always from $\S{T}$ and therefore may have arbitrary length. We assume no restrictions on the elements of the time series. The elements can be real values, feature vectors, symbols, trees, graphs, and mixtures thereof. 

This contribution presents sufficient conditions for existence of a sample mean in restricted and unrestricted form. We show that common DTW distances mentioned in the literature satisfy the proposed sufficient conditions. A key result is the Reduction Theorem stating that there is a sample-dependent bound $\rho$ on the length beyond which the Fr\'echet function can not be further decreased. For the two sample time series in Figure \ref{fig:ex_unrestricted_mean} the bound is $\rho = 4$. Hence, the restricted sample mean $z_4$ is also an unrestricted sample mean. 

This contribution has the following implications: Existence of a sample mean together with the necessary conditions of optimality proposed in \cite{Schultz2017} enable the formulation of exact mean algorithms \cite{Brill2017}. Existence of restricted sample means theoretically justify prior work \cite{Cuturi2017,Hautamaki2008,Petitjean2011,Rabiner1979,Schultz2017,Soheily-Khah2016} in the sense that the concept of a restricted sample mean is not a phantom but does in fact exist. Existence of the weighted mean justfies the soft-DTW approach proposed by \cite{Cuturi2017}. Finally, this contribution is a first step towards a statistical theory of DTW spaces in the spirit of a statistical theory of shape, tree, and graph spaces \cite{Dryden1998,Feragen2013,Ginestet2012,Huckemann2010,Jain2016,Kendall1984,Marron2014}.

The rest of this paper is structured as follows: Section 2 states the main results of this contribution and Section 3 concludes with a summary of the main findings and an outlook to further research. All proofs are delegated to the appendix.

\section{Existence of a Sample Mean via the Reduction Theorem}\label{sec:reduction-theorem+implications}

This section first introduces the DTW-distance and Fr\'echet functions. Then the Reduction Theorem is stated and its implications are presented. Finally, sufficient conditions of existence of a sample mean are proposed.

\paragraph*{Notations.} We write $\R_{\geq 0}$ for the set of non-negative reals. By $\N$ we denote the set of positive integers. We write $[n]$ to denote the set $\cbrace{1, \ldots, n}$ for a given $n \in \N$. Finally, $\S{S}^N = \S{S} \times \cdots \times \S{S}$ is the $N$-fold Cartesian product of the set $\S{S}$, where $N \in \N$.

\subsection{The Dynamic Time Warping Distance}

Suppose that $\S{A}$ is an attribute set. A \emph{time series} $x$ of \emph{length} $\ell(x) = m$ is a sequence $x = (x_1, \ldots, x_m)$ consisting of \emph{elements} $x_i \in \S{A}$ for every \emph{time point} $i \in [m]$. By $\S{T}_n$ we denote the set of all time series of length $n \in \N$ with elements from $\S{A}$. Then 
\[
\S{T} = \bigcup_{n \in \N} \S{T}_n
\] 
is the set of all time series of finite length with elements from $\S{A}$. 

Without further mention, we assume that the attribute set $\S{A}$ is given. Since we do not impose restrictions on the attribute set $\S{A}$, the above definition of time series covers a broad range of sequential data structures. For example, to represent real-valued univariate and multivariate time series, we use $\S{A} = \R$ and $\S{A} = \R^d$, resp., as attribute set. For text strings and biological sequences, the set $\S{A}$ is an alphabet consisting of a finite set of symbols. Further examples are time series of satellite images and time series of graphs as studied in anomaly detection. 

\medskip

Time series vary in length and speed. To filter out these variations, we introduce the technique of dynamic time warping.
\begin{definition}\label{definition:warping-path}
Let $m, n \in \N$. A \emph{warping path} of order $m \times n$ is a sequence $p = (p_1 , \dots, p_L)$ of $L$ points $p_l = (i_l,j_l) \in [m] \times [n]$ such that
\begin{enumerate}
\item $p_1 = (1,1)$ and $p_L = (m,n)$ \hfill\emph{(\emph{boundary conditions})}
\item $p_{l+1} - p_{l} \in \cbrace{(1,0), (0,1), (1,1)}$ for all $l \in [L-1]$ \hfill\emph{(\emph{step condition})} 
\end{enumerate}
\end{definition}
The set of all warping paths of order $m \times n$ is denoted by $\S{P}_{m,n}$. A warping path of order $m \times n$ can be thought of as a path in a $[m] \times [n]$ grid, where rows are ordered top-down and columns are ordered left-right. The boundary condition demands that the path starts at the upper left corner and ends in the lower right corner of the grid. The step condition demands that a transition from on point to the next point moves a unit in exactly one of the following directions: down, diagonal, and right. 

A warping path $p = (p_1, \ldots, p_L)\in \S{P}_{m,n}$ defines an alignment (or warping) between time series $x = (x_1, \ldots, x_m)$ and $y = (y_1, \ldots, y_n)$. Every point $p_l = (i_l,j_l)$ of warping path $p$ aligns element $x_{i_l}$ to element $y_{j_l}$. The \emph{cost} of aligning time series $x$ and $y$ along warping path $p$ is defined by
\begin{equation*}
c_p(x,y) = \sum_{l=1}^L d\args{x_{i_l}, y_{j_l}},
\end{equation*}
where $d: \S{A} \times \S{A} \rightarrow \R$ is a \emph{local distance function} on $\S{A}$. We demand that the local distance $d$ satisfies the following properties:
\begin{enumerate}
\itemsep0em
\item $d(a, a') \geq 0$
\item $d(a, a) = 0$
\item $d(a, a') = d(a', a)$
\end{enumerate}
for all $a, a'\in \S{A}$. As with the attribute set $\S{A}$ we tacitly assume that the local distance $d$ is given without further mention.

Now we are in the position to define the DTW-distance. We obtain the DTW-distance between two time series $x$ and $y$ by minimizing the cost $c_p(x,y)$ over all possible warping paths. 
\begin{definition}\label{def:DTW-distance}
Let $f:\R_{\geq 0} \rightarrow \R$ be a monotonous function. Let $x$ and $y$ be two time series of length $m$ and $n$, respectively. The \emph{DTW-distance} between $x$ and $y$ is defined by
\begin{equation*}
\dtw(x,y) = \min \cbrace{f\args{c_p(x,y)} \,:\, p \in \S{P}_{m,n}}.
\end{equation*}
An \emph{optimal warping path} is any warping path $p \in \S{P}_{m,n}$ satisfying $\dtw(x, y) = f\args{c_p(x,y)}$.
\end{definition}

The next example presents a common and widely applied DTW-distance in order to illustrates all components of Definition \ref{def:DTW-distance}.

\begin{example}\label{ex:Euclidean-DTW} The \emph{Euclidean DTW-distance} is specified by the attribute set $\S{A} = \R^d$, the squared Euclidean distance $d(x, y) = \normS{x-y}{^2}$ for all $x, y \in \S{A}$, and the square root function $f(x) = \sqrt{x}$ for all $x \in \R_{\geq 0}$.
\end{example}

Even if the underlying local distance function $d$ is a metric, the induced DTW-distance is generally only a pseudo-semi-metric satisfying 
\begin{enumerate}
\item $\dtw(x, y) \geq 0$
\item $\dtw(x, x) = 0$
\end{enumerate}
for all $x, y \in \S{T}$. Computing the DTW-distance and deriving an optimal warping path is usually solved by applying techniques from dynamic programming \cite{Sakoe1978}.

A \emph{DTW-space} is a pair $\args{\S{T}, \dtw}$ consisting of a set of time series of finite length and a DTW-distance $\delta$ defined on $\S{T}$. For the sake of convenience, we occasionally write $\S{T}$ to denote a DTW-space and tacitly assume that $\dtw$ is the underlying DTW-distance.

\subsection{Fr\'echet Functions}

Let $\args{\S{T}, \dtw}$ be a DTW-space. A \emph{loss function} is a monotonously increasing function of the form $h: \R_{\geq 0} \rightarrow \R$. A typical example of a loss function is the squared loss $h(u) = u^2$ for all $u \geq 0$.

\begin{definition}\label{definition:Frechet-Function}
Let $\S{X} = \args{x^{(1)}, \ldots, x^{(N)}} \in \S{T}^N$ be a sample of $N$ time series $x^{(k)}$ with corresponding loss function $h_k: \R_{\geq 0} \rightarrow \R$ for all $k \in [N]$. Then the function 
\begin{align*}
F: \S{T} \rightarrow \R, \quad x \mapsto \frac{1}{N}\sum_{i = 1}^N h_k\args{\dtw\args{x,x^{(k)}}}
\end{align*}
is the \emph{Fr\'echet function} of sample $\S{X}$ corresponding to the loss functions $h1, \ldots, h_N$. 
\end{definition}

We omit explicitly mentioning the corresponding loss functions of a Fr\'echet function if no confusion can arise. Note that Definition \ref{definition:Frechet-Function} refers to the unrestricted form as described in Section \ref{sec:intro}. We present some examples and assume that $\S{X} = \args{x^{(1)}, \ldots, x^{(N)}} \in \S{T}^N$ is a sample of $N$ time series.

\begin{example}\label{ex:Fp}
Let $p \geq 1$. The Fr\'echet function of $\S{X}$ corresponding to the loss functions $h_k(u) = u^p$ takes the form
\begin{align*}
F^p(x) = \sum_{k = 1}^N \dtw^p\args{x,x^{(k)}}.
\end{align*}
For $p = 1$ \emph{($p=2$)} the Fr\'echet function $F^p$ generalizes the concept of sample median (sample mean) in Euclidean spaces. 
\end{example}

\begin{example}\label{ex:wFp}
Let $w_k > 0$. The Fr\'echet function of $\S{X}$ corresponding to the loss functions $h_k(u) = w_k \cdot u^p$ is of the form
\begin{align*}
F^p(x) = \sum_{k = 1}^N w_k\dtw^p\args{x,x^{(k)}}.
\end{align*}
The function $F^p(x)$ is a weighted average of the sum of $p$-distances $\delta^p$. In the special case of $w_k = 1/N$, the function $F^p(x)$ averages the sum of $p$-distances.
\end{example}

Next, we consider the global minimizers -- if exist -- of a Fr\'echet function.

\begin{definition}
The \emph{sample mean set} of $\S{X}$ is the (possibly empty) set defined by 
\[
\S{F} = \cbrace{z \in \S{T} \,:\, F(z) \leq F(x) \text{ for all } x \in \S{T}}.
\]
The elements of $\S{F}$ are the \emph{sample means} of $\S{X}$. 
\end{definition}
A sample mean is a time series that minimizes the Fr\'echet function $F$. It can happen that the corresponding set $\S{F}$ is empty. In this case, a sample mean does not exist. Existence of a sample mean depends on the choice of DTW-distance and loss function. Moreover, if a sample mean exists, it may not be uniquely determined. In contrast, the \emph{sample variance}\footnote{The sample variance $F^*$ corresponding to loss $h(u) = u^2$ generalizes the sample variance in Euclidean spaces.} 
\[
F^* = \inf_{x \in \S{T}} F(x)
\]
exists and is uniquely determined, because the DTW-distance is bounded from below and the loss is monotonously increasing. Thus, existence of a sample mean means that the Fr\'echet function $F$ attains its infimum. 

The next example presents a DTW-space for which a sample mean may not always exist. This example is inspired by the edit distance for sequences and drastically simplified to directly convey the main idea.
\begin{example}\label{example:non-existence}
Let $\S{A} = \R$ be the attribute set with local cost function $d$ of the form
\[
d(a, a') = 
\begin{cases}
\argsS{a-a'}{^2} & a \neq 0 \text{ \emph{and} } a' \neq 0\\
1		   & a = 0 \text{ \emph{xor} } a' = 0\\
0		   & a = 0 \text{ \emph{and} } a' = 0  
\end{cases}
\] 
for all $a, a' \in \S{A}$. We assume that the function $f$ corresponding to the DTW distance is the identity such that
\[
\dtw(x, y) = \min \cbrace{c_p(x,y) \,:\, p \in \S{P}_{m,n}},
\]
for all time series $x$ and $y$ of length $m$ and $n$, respectively. Consider the sample $\S{X} = (x, y)$ consisting of two time series $x = (1, 1)$ and $y = (1, -1)$. As indicated by Figure \ref{fig:ex_nonexistence}, the Fr\'echet function 
\[
F(z) = \frac{1}{2}\delta(x, z) + \frac{1}{2}\delta(y,z) 
\]
never attains its infimum $0.5$. Thus $F(x)$ has no global minimum and therefore $\S{X}$ has no sample mean. 
\end{example}

\begin{figure}[t]
\centering
\includegraphics[width=0.99\textwidth]{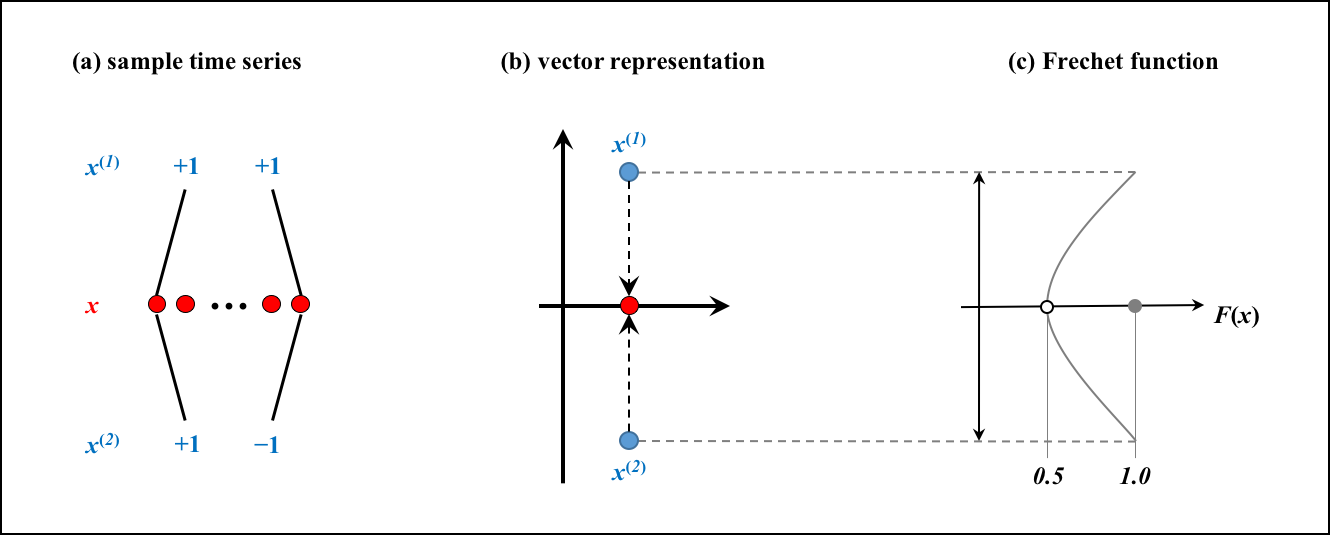}
\caption{Example of non-existence of a sample mean. Plot (a) shows a sample $\S{X} = \args{x^{(1)}, x^{(2)}}$ of two time series and a candidate solution $x$ of the Fr\'echet function $F$. The black lines indicate optimal warping paths between the sample time series and $x$. Any warping path must satisfy the boundary condition. Therefore, it is sufficient to consider the case that $x$ is of length $2$. Plot (b) depicts the sample time series and candidate solution $x$ of length $2$ as points in the vector space $\R^2$. Plot (c) shows the Fr\'echet function $F(x)$ of $\S{X}$ corresponding to the DTW distance of Example \ref{example:non-existence}. Suppose $(z_i)$ is a sequence of candidate solutions starting at $z_1 = x$ and converging to $x$ on a straight line as indicated by the dashed arrow in plot (b). Then the induced sequence $(F(z_i))$  is strictly monotonously decreasing and converges to but never attains its infimum $0.5$. Hence, the Fr\'echet function $F$ has not a global minimum, which implies that $\S{X}$ has not a sample mean.}
\label{fig:ex_nonexistence}
\end{figure}

\commentout{
Figure \ref{fig:ex_infimum} schematically depicts an example of a bounded function that does not attain its infimum and therefore does not have a global minimum. 
\begin{figure}[t]
\centering
\includegraphics[width=0.5\textwidth]{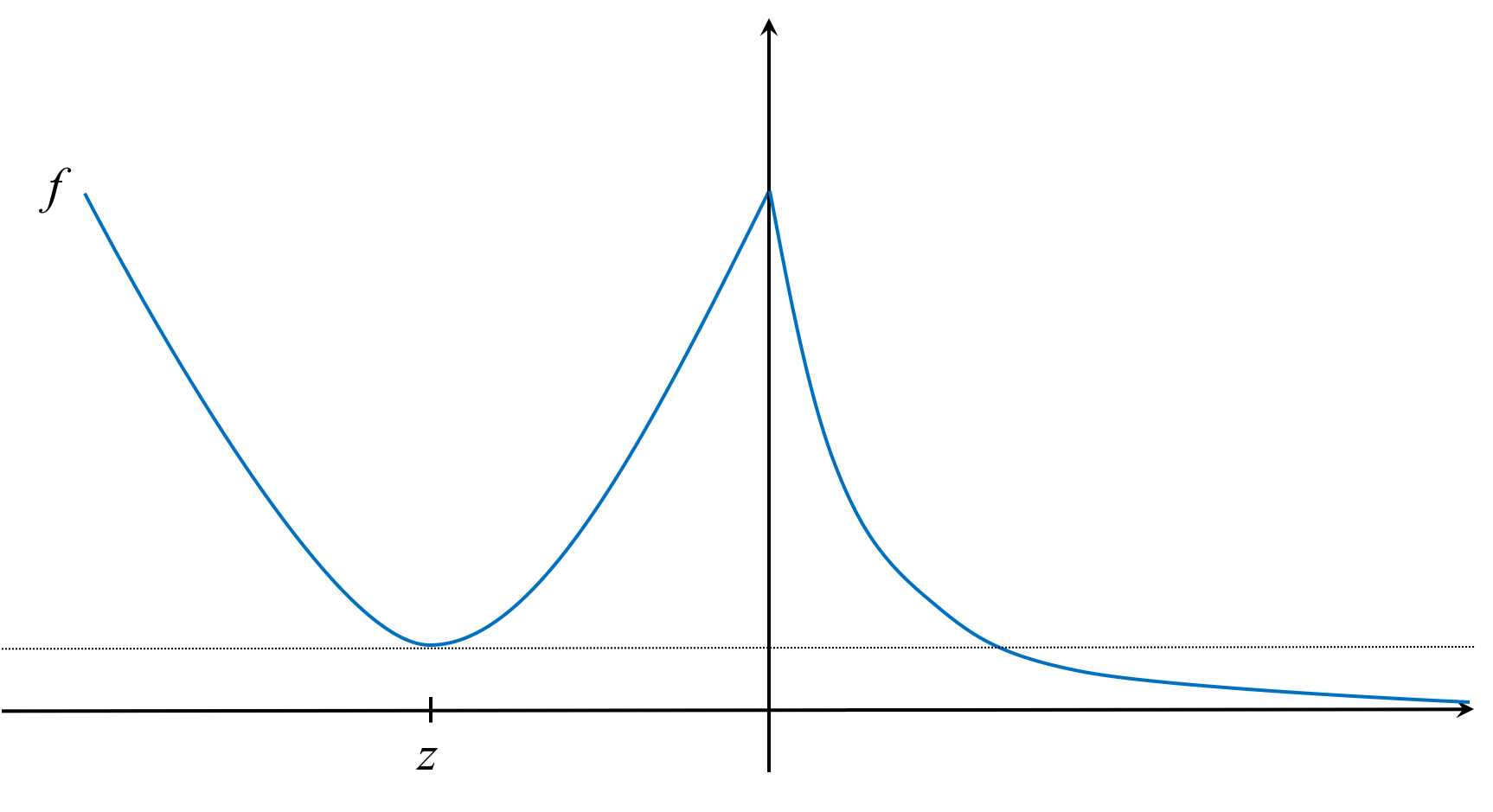}
\caption{Example of a function $f$ that does not attain its infimum and therefore does not have a minimum. The function has a local minimum at $z$ with $f(z) > 0$. The function is strictly monotonously decreasing for $x \geq 0$ and asymptotically converges but never attains its infimum zero. Hence, $f$ does not have a global minimum.}
\label{fig:ex_infimum}
\end{figure}
}

\subsection{Restricted and Unrestricted Fr\'echet Functions}

The Fr\'echet function $F:\S{T} \rightarrow \R$ of Definition \ref{definition:Frechet-Function} is in unrestricted form, because it is defined on the entire set $\S{T}$ and imposes no restrictions on the length of the sample mean. In restricted form, the function 
\[
F_m: \S{T}_m \rightarrow \R, \quad x \mapsto F(x).
\]
is the Fr\'echet function of $\S{X}$ restricted to the subset $\S{T}_m \subset \S{T}$ of all time series of length $m$. It is important to note that the lengths of the sample time series in $\S{X}$ may vary, but the length of the independent variable $x$ of $F_m(x)$ is fixed beforehand to value $m$. In line with the unrestricted form, the set 
\[
\S{F}_m = \cbrace{z \in \S{T}_m \,:\, F_m(z) \leq F_m(x) \text{ for all } x \in \S{T}_m}
\]
is the \emph{restricted sample mean set} of $\S{X}$ restricted to the subset $\S{T}_m$. Occasionally, we call the elements of $\S{F}_m$ restricted sample means and the elements of $\S{F}$ the (unrestricted) sample means. As for the unrestricted form, existence of a restricted sample mean depends on the choice of DTW-distance and loss function.

\subsection{The Reduction Theorem}

This section presents sufficient conditions for existence of a sample mean in restricted and unrestricted form. The approach is as follows: First, we present sufficient conditions for existence of restricted sample means. Second, under these assumptions we infer that an unrestricted sample mean also exists. The main tool for the second step is the Reduction Theorem. Proofs are delegated to the appendix.

\medskip

Suppose that $\S{X} = \args{x^{(1)}, \ldots, x^{(N)}} \in \S{T}^N$ is a sample of $N$ time series. The Reduction Theorem is based on the notion of \emph{reduction bound} $\rho(\S{X})$ of sample $\S{X}$. The exact definition of $\rho(\S{X})$ requires some technicalities and is fully spelled out in Section \ref{subsec:glued-graphs}. Here, we present a simpler definition that conveys the main idea and covers the relevant use cases in pattern recognition. For this, we assume that every sample time series $x^{(k)}$ of sample $\S{X}$ has length $\ell(x^{(k)}) \geq 2$. Then the \emph{reduction bound} of $\S{X}$ is defined by
\begin{align}\label{eq:reduction-bound-simple}
\rho(\S{X}) = \sum_{k=1}^N \ell\args{x^{(k)}} - \; 2(N-1).
\end{align}
In contrast to Eq.~\eqref{eq:reduction-bound-simple}, the exact definition of $\rho(\S{X})$ admits samples that contain trivial time series of length one. Equation \eqref{eq:reduction-bound-simple} shows that the reduction bound of a sample increases linearly with the sum of the lengths of the sample time series. The following results hold for arbitrary samples and assume the exact definition of a reduction bound as provided in Section \ref{subsec:glued-graphs}.

\begin{theorem}[Reduction Theorem]\label{theorem:reduction}
Let $F$ be the Fr\'echet function of a sample $\S{X}\in \S{T}^N$. Then for every time series $x \in \S{T}$ of length $\ell(x) > \rho(\S{X})$ there is a time series $x' \in \S{T}$ of length $\ell(x') = \ell(x) -1$ such that $F(x') \leq F(x)$. 
\end{theorem}

The Reduction Theorem deserves some explanations. To illustrate the following comments we refer to Figures \ref{fig:ex_reduce01} -- \ref{fig:ex_longmean} with the following specifications: In these figures, we assume univariate time series with real values. The underlying distance is the Euclidean DTW-distance of Example \ref{ex:Euclidean-DTW}. The Fr\'echet functions of the different samples $\S{X} = \args{x^{(1)}, x^{(2)}}$ are given by 
\[
F(z) = \frac{1}{2}\delta\args{x^{(1)}, z} + \,\frac{1}{2}\delta\args{x^{(1)}, z}.
\]
The figures show warping paths by black lines connecting aligned elements of the time series to be compared. We make the following observations:

\paragraph*{\textmd{1.}} From the proof of the Reduction Theorem follows that every candidate solution $x$ whose length exceeds the reduction bound has an element that can be  removed without increasing the value $F(x)$. Such elements are said to be \emph{redundant}. Figure \ref{fig:ex_reduce01} schematically characterizes redundant elements of a time series. 

\paragraph*{\textmd{2.}} In general, removing a redundant element does not increase the Fr\'echet function. Figure \ref{fig:ex_reduce01} shows that removing a redundant element can even decrease the value of the Fr\'echet function.

\paragraph*{\textmd{3.}} The reduction bound of the sample in Figure \ref{fig:ex_reduce01} is given by
\[
\rho(\S{X}) = \ell\args{x^{(1)}}+\;\ell\args{x^{(2)}} - \;2(N-1) = 4 + 4 - 2 = 6.
\]
The length of time series $x$ is only $\ell(x) = 5 < \rho(\S{X})$. This shows that short candidate solutions whose lengths are bounded by the reduction bound may also have redundant elements that can be removed without increasing the value of the Fr\'echet function. Existence of a redundant element depends on the choice of warping path between $x$ and the sample time series. For short time series $x$, we can always find warping paths such that $x$ has no redundant elements. In contrast, long time series whose lengths exceed the reduction bound always have a redundant element, regardless which warping paths we consider.

\paragraph*{\textmd{4.}} Removing a non-redundant element of a candidate solution can increase the value of the Fr\'echet function. Figure \ref{fig:ex_reduce02} presents an example. 

\paragraph*{\textmd{5.}} The Reduction Theorem does not exclude existence of sample means whose lengths exceed the reduction bound of a sample. Figure \ref{fig:ex_longmean} presents an example for which a sample mean can have almost any length.

\bigskip

The Reduction Theorem and observations 1--4 form the basis for existence proofs in unrestricted form and point to a technique to improve algorithms for approximating a sample mean (if exists). Statements on the existence of a sample mean are presented in the next section. From observations 1--3 follows that a candidate solution $x$ of any length could be improved or at least shortened by detecting and removing redundant elements of $x$. This observation is not further explored in this article and left for further research.

\begin{samepage}
\begin{figure}[t]
\centering
\includegraphics[width=0.8\textwidth]{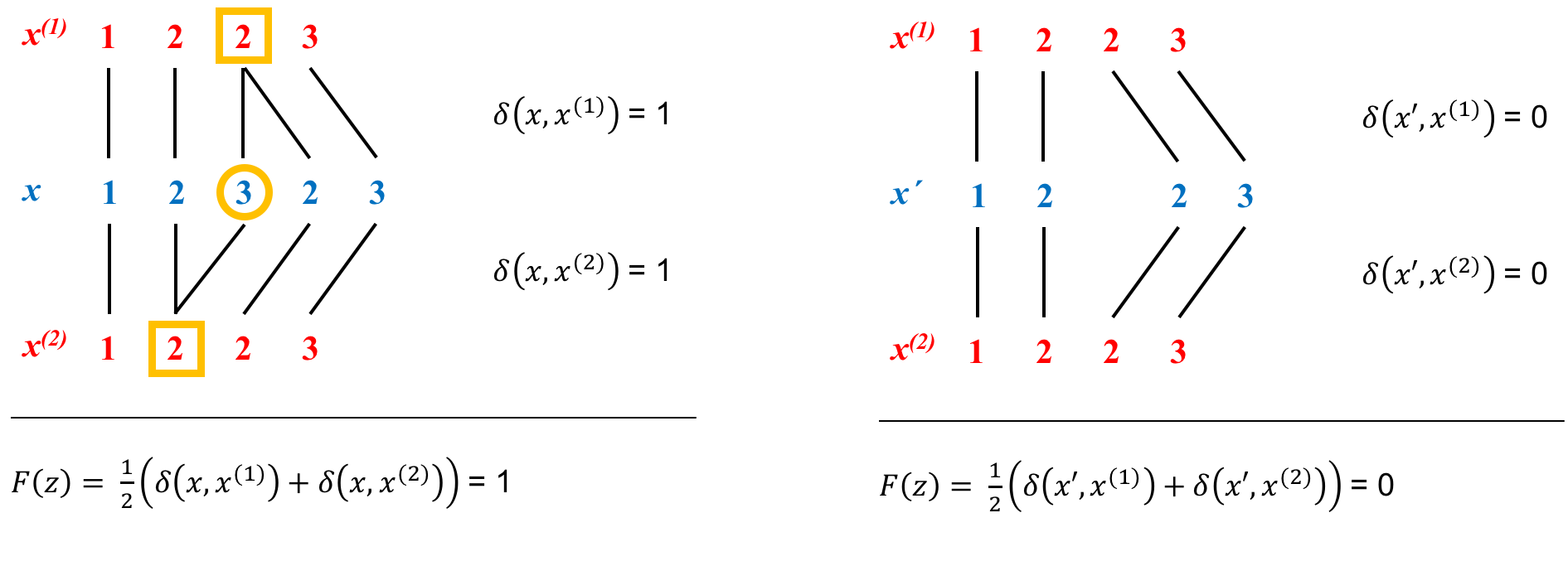}
\vspace{-0.5cm}
\caption{Schematic depiction of redundant elements. Shown are two sample time series $x^{(1)}$ and $x^{(2)}$ (red) and two further time series $x$ and $x'$ (blue). The third element of $x$ is redundant (enclosed by a circle). Redundant elements are characterized by the following property: An element $x_i$ of time series $x$ is redundant if every element of the sample time series connected to $x_i$ is also connected to another element of $x$. The elements of the sample time series connected to the third element of $x$ are enclosed by a square. Both squared elements are connected to two elements of $x$. The time series $x'$ is obtained from $x$ by removing its redundant element. The DTW-distances of the sample time series from $x$ are both one and from $x'$ are both zero (see main text for a specification of the DTW distance). Then we obtain $F(x) > F(x')$. This shows that removing the redundant element in $x$ does not increase the value of the Fr\'echet function. In this particular case, the value of the Fr\'echet function is even decreased.}
\label{fig:ex_reduce01}
\vspace{-0.2cm}
\rule{\textwidth}{0.5pt}
\end{figure}
\begin{figure}[h!]
\centering
\includegraphics[width=0.75\textwidth]{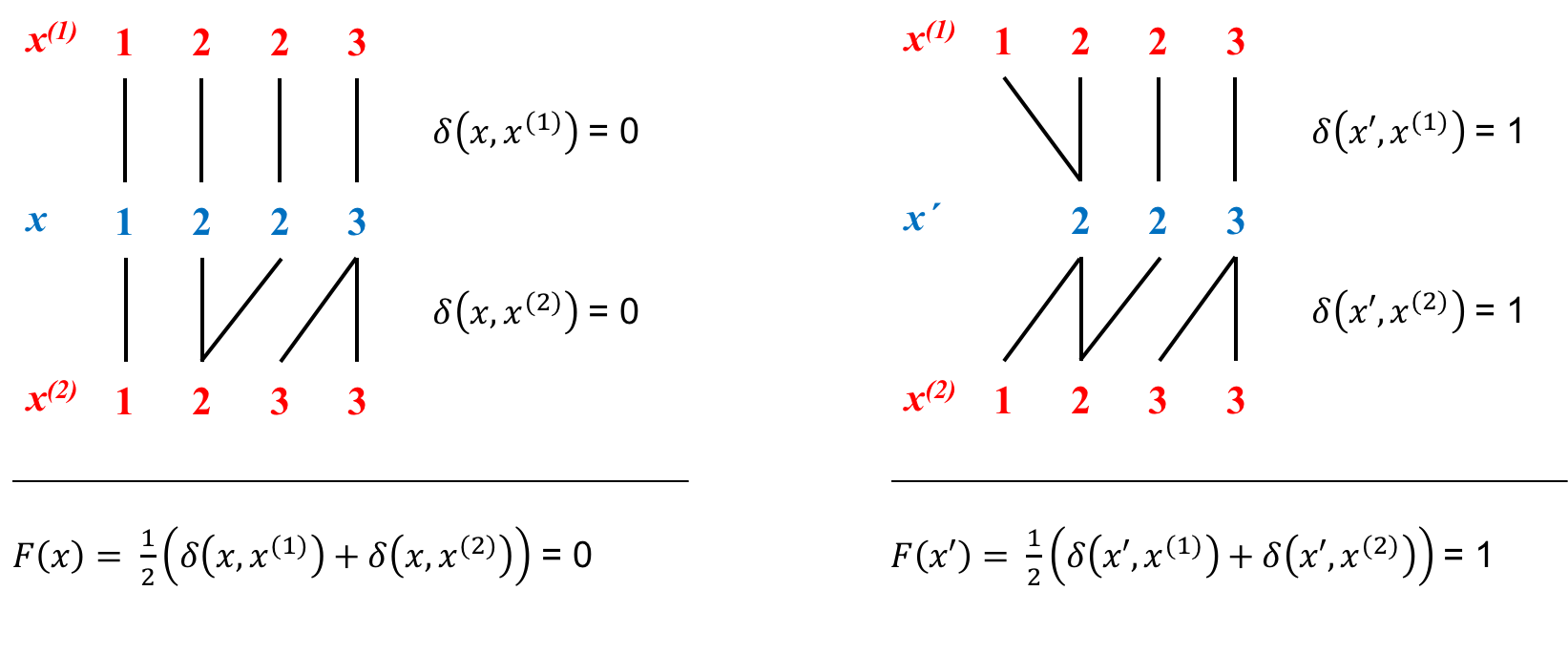}
\vspace{-0.5cm}
\caption{Removing a non-redundant element can increase the Fr\'echet function. The time series $x'$ is obtained from $x$ by removing the first element. The first element of $x$ is not redundant according to the characterization of redundant elements given in Figure \ref{fig:ex_reduce01}. The DTW-distances of the sample time series from $x$ are both zero and from $x'$ are both one (see main text for a specification of the DTW distance). This shows that $F(x) < F(x')$.}
\label{fig:ex_reduce02}
\vspace{-0.2cm}
\rule{\textwidth}{0.5pt}
\end{figure}
\begin{figure}[h!]
\centering
\includegraphics[width=0.65\textwidth]{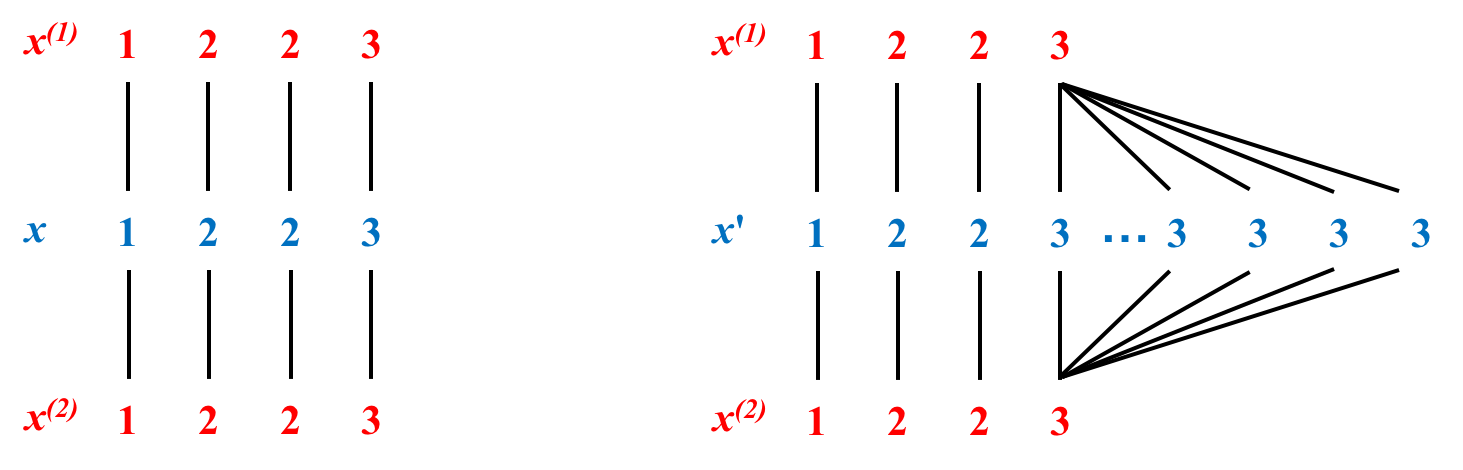}
\vspace{-0.2cm}
\caption{Restricted sample means of almost any length (see main text for a specification of the DTW distance). The time series $x$ is a sample mean of $\S{X}$, because $F(x) = 0$ is the lowest possible value. The time series $x'$ is obtained from the sample mean $x$ by appending arbitrarily many time points with element $3$. From $F(x')=0$ follows that $x'$ is also a sample mean of $\S{X}$.}
\label{fig:ex_longmean}
\end{figure}
\end{samepage}


\subsection{Sufficient Conditions of Existence}

In this section, we derive sufficient conditions of existence of a sample mean in restricted and unrestricted form. 

\medskip

The Reduction Theorem guarantees existence of a sample mean in unrestricted form if sample means exist in restricted forms. Thus, existence proofs in the general unrestricted form reduce to existence proofs in the simpler restricted form. This statement is proved in Corollary \ref{cor:existence:restricted-to-unrestricted}.

\begin{corollary}\label{cor:existence:restricted-to-unrestricted}
Let $\S{X}\in \S{T}^N$ be a sample and let $\rho \in \N$ be the reduction bound of $\S{X}$. Suppose that $\S{F}_m \neq \emptyset$ for every $m \in \bracket{\rho}$. Then $\S{X}$ has a sample mean.
\end{corollary}

It is not self-evident that existence of a class of restricted sample mean implies existence of a sample mean. To see this, we define the restricted (sample) variance by
\[
F_m^* = \inf_{x \in \S{T}_m} F_m(x).
\]
If $F_m$ attains its infimum, then $\S{X}$ has a restricted sample mean. Suppose that $\S{X}$ has a restricted sample mean for every $m \in \N$. Then 
\[
v_m = \min_{l \leq m} F_m^*.
\]
is the smallest restricted variance over all lengths $l \in [m]$. The sequence $(v_m)_{m \in\N}$ is bounded from below and monotonously decreasing. Therefore, the sequence $(v_m)_{m \in\N}$ converges to the unrestricted sample variance $F_*$. Then $\S{X}$ has a sample mean only if the sequence $(v_m)_{m \in\N}$ attains its infimum $F_*$. Corollary \ref{cor:existence:restricted-to-unrestricted} guarantees that the sequence $(v_m)_{m \in\N}$ indeed attains its infimum $F^*$ latest at $m = \rho(\S{X})$.

\medskip

Next, we present sufficient conditions of existence. The first result proposes sufficient conditions of existence of a restricted and unrestricted sample mean for time series (sequences) with discrete attribute values.
\begin{proposition}\label{prop:restricted-existence-discrete-case}
Let $\S{X} \in \S{T}^N$ be a sample. Suppose that $\S{A}$ is a finite attribute set. Then the following statements hold:
\begin{enumerate}
\itemsep0em
\item $\S{F}_m \neq \emptyset$ for every $m\in \N$. 
\item $\S{F}\neq \emptyset$.
\end{enumerate}
\end{proposition}

The second result proposes sufficient conditions of existence of a restricted and unrestricted sample mean of uni- and multivariate time series with elements from $\S{A} = \R^d$. 
\begin{proposition}\label{prop:sufficient-conditions-of-existence}
Let $\S{X} \in \S{T}^N$ be a sample. Suppose that the following assumptions hold:
\begin{enumerate}
\item $\args{\S{A}, d}$ is a metric space of the form $\args{\R^d, \norm{\cdot}}$, where $\norm{\cdot}$ is a norm on $\R^d$. 
\item The loss functions $h_1, \ldots, h_N$ are continuous and strictly monotonously increasing.
\end{enumerate}
Then the following statements hold:
\begin{enumerate}
\itemsep0em
\item $\S{F}_m \neq \emptyset$ for every $m\in \N$. 
\item $\S{F}\neq \emptyset$.
\end{enumerate}
\end{proposition}

The attribute set $\S{A}$ covers the case of univariate ($d=1$) and the case of multivariate ($d>1$) time series. The local cost function $d$ on $\S{A}$ is a metric induced by a norm on $\R^d$. Loss functions $h:\R_{\geq 0} \rightarrow \R$ of the form $h(u) =  w\cdot u^p$ are continuous and strictly monotonously increasing for $w > 0$ and $p \geq 1$. Thus, the sufficient conditions of Prop.~\ref{prop:sufficient-conditions-of-existence} cover customary DTW-spaces. We conclude this section with a remark on weighted means. 

\begin{remark}\label{cor:existence-of-weighted-mean}
Proposition \ref{prop:sufficient-conditions-of-existence} holds when we replace the loss functions $h_k$ by the loss functions $h'_k = w_k h_k$ with $w_k \in \R_{\geq 0}$ for all $k \in [N]$. \end{remark}

Note that only loss functions $h'_k = w_k h_k$ with positive weights $w_k > 0$ are strictly monotonously increasing. Hence, assumption (2) of Prop.~\ref{prop:sufficient-conditions-of-existence} is violated for loss functions $h_k'$ with zero-weights $w_k = 0$. Corollary \ref{cor:existence-of-weighted-mean} relaxes the condition of strictly positive weights to non-negative weights. For a proof of Remark \ref{cor:existence-of-weighted-mean} we refer to Section \ref{subsec:proofs}.

\section{Conclusion}
This article presents sufficient conditions for the existence of a sample mean in DTW spaces in restricted and unrestricted form. The sufficient conditions hold for common DTW distances reported in the literature. Key result is the Reduction Theorem stating that time series whose lengths exceed the reduction bound can be reduced to shorter time series without increasing the value of the Fr\'echet function. This result guarantees existence of a sample mean in unrestricted form if sample means exist in restricted form. The proof of the Reduction Theorem is framed into the theory of warping graphs. The existence proofs theoretically justify existing mean-algorithms and related pattern recognition applications in retrospect. The Reduction Theorem sets the stage for studying the unrestricted sample mean problem. Finally, existence of the sample mean sets the stage for constructing exact algorithms and a statistical theory of DTW spaces. The next step towards such a theory consists in studying under which conditions a sample mean is a consistent estimator of a population mean.

\bigskip

\paragraph*{\em\textbf{Acknowledgements.}} B.~Jain was funded by the DFG Sachbeihilfe \texttt{JA 2109/4-1}.

\begin{appendix}
\section{Theory of Warping Graphs}
\newtheorem*{theorem*}{Theorem}
\newtheorem*{proposition*}{Proposition}
\newtheorem*{corollary*}{Corollary}
\newtheorem*{lemma*}{Lemma}
\newtheorem*{remark*}{Remark}

This appendix develops a theory of warping graphs to prove the Reduction Theorem and all other results in this article. The line of argument follows a bottom-up approach. First, we introduce  warping chains to model abstract warping paths as given in Definition \ref{definition:warping-path} and derive their relevant local properties. Then we proceed to warping graphs that model the alignment of two time series by a warping path and derive global properties from local properties. We enhance warping graphs with node labels and define the notion of weight of a warping graph to model the DTW-distance. Finally, we glue warping graphs to model Fr\'echet functions, derive the Reduction Theorem, and prove the statements presented in Section \ref{sec:reduction-theorem+implications}.

\subsection{Warping Chains}

The basic constituents of a Fr\'echet function are time series and optimal warping paths. This section represents the linear order of time series by chains. Then we introduce the notion of warping chain to model abstract warping paths and study its local properties.

\medskip

Let $\S{V}$ be a partially ordered set with partial order $\leq_{\S{V}}$. Suppose that $i,j \in \S{V}$ are two elements with $i \leq_{\S{V}} j$. We write $i <_{\S{V}} j$ to mean $i \leq_{\S{V}} j$ and $i \neq j$. A linear order $\leq_{\S{V}}$ on $\S{V}$ is a partial order such that any two elements in $\S{V}$ are comparable: For all $i, j \in \S{V}$, we have either $i \leq_{\S{V}} j$ or $j \leq_{\S{V}} i$.

\begin{definition}
A \emph{chain} is a linearly ordered set. 
\end{definition}
A chain $\S{V}$ models the order of a time series $x = (x_1, \ldots, x_m)$, where element $i \in \S{V}$ refers to the positions of element $x_i$ in $x$. For the sake of convenience, we assume that the explicit notation of a chain $\S{V} = \cbrace{i_1, \ldots, i_m}$ always lists its elements in linear order $i_1 <_{\S{V}} \cdots <_{\S{V}} i_m$. We call $i_1$ the first and $i_m$ the last element in $\S{V}$. The first and last element of a chain are the boundary elements. Any element of chain $\S{V}$ that is not a boundary element is called an inner element of $\S{V}$.

A subset $\S{V}' \subseteq \S{V}$ is a subchain of $\S{V}$. Note that any subset of a chain is again a chain by transitivity of the linear order. Suppose that $i_p, i_q \in \S{V}$ such that $i_p \leq_{\S{V}} i_q$. Then the subchain $\S{V}' = \cbrace{i \in \S{V} \,:\, i_p \leq_{\S{V}} i \leq_{\S{V}} i_q}$ is said to be contiguous.

Let $\S{V} = \cbrace{i_1, \ldots, i_m}$ be a chain and let $\S{V}^* = \S{V} \cup \cbrace{*}$, where $*$ is a distinguished symbol denoting the void element. The successor $i_l^+$ and predecessor $i_l^-$ of element $i_l \in \S{V}$ are defined by
\begin{align*}
i_l^+ = \begin{cases}
i_{l+1} & 1 \leq l < L \\
* & l = L
\end{cases}
& \qquad \text{ and } \qquad
i_l^- = \begin{cases}
i_{l-1} & 1 < l \leq L \\
* & l = 1
\end{cases}.
\end{align*}

We assume that $\S{W}$ is another chain. The chains $\S{V}$ and $\S{W}$ induce a partial order on the product $\S{U} = \S{V} \times \S{W}$ by
\[
(i,j) \leq_{\S{U}} (r, s) \; \Leftrightarrow \; i \leq_{\S{V}} r \text{ and } j \leq_{\S{W}} s
\]
for all $(i,j), (r, s) \in \S{U}$. 
\begin{definition}
Let $\S{U} = \S{V} \times \S{W}$ be the product of chains $\S{V}$ and $\S{W}$. The \emph{successor map} on $\S{U}$ is a point-to-set map
\[
S_{\S{U}}: \S{U} \rightarrow 2^{\S{U}}, \quad (i,j) \;\mapsto\;
\cbrace{\args{i^+, j}, \args{i, j^+}, \args{i^+, j^+}} \; \cap \; \big(\S{V} \times \S{W}\big),
\]
where $2^{\S{U}}$ denotes the set of all subsets of $\S{U}$. 
\end{definition}
The successor map models the set of feasible warping steps for a given element $(i,j) \in \S{U}$. Intersection of the successor map with $\S{V} \times \S{W}$ ensures that elements with $i^+ = *$ or $j^+=*$ are excluded. The successor map sends $(i,j)$ to the empty set if $i$ and $j$ are the last elements of the respective chains $\S{V}$ and $\S{W}$. The next result shows that the successor map preserves the partial product order $\leq_{\S{U}}$ as well as the linear orders $\leq_{\S{V}}$ and $\leq_{\S{W}}$.

\begin{lemma}\label{lemma:order-preserving}
Let $\S{U} = \S{V} \times \S{W}$ be the product of chains $\S{V}$ and $\S{W}$. Suppose that $e = (i,j) \in \S{U}$ is an element with $S_{\S{U}}(e) \neq \emptyset$. Then the following order preserving properties hold:
\begin{enumerate}
\item $e \leq_{\S{U}} e'$ for all $e'\in S_{\S{U}}(e)$.
\item $i \leq_{\S{V}} r \text{ and } j \leq_{\S{W}} s$ for all $(r,s)\in S_{\S{U}}(i,j)$.
\end{enumerate}
\end{lemma}

\begin{proof} Directly follows from the definitions of $\leq_{\S{U}}$ and $S_{\S{U}}$.
\end{proof}

\begin{lemma}\label{lemma:warping-chain}
Let $\S{U} = \S{V} \times \S{W}$ be the product of chains $\S{V}$ and $\S{W}$. Let $\S{E} \subseteq \S{U}$ be a subset consisting of $L$ elements $e_1, \ldots, e_L \in \S{E}$ such that $e_{l+1} \in S_{\S{U}}(e_l)$ for all $l \in [L-1]$. Then $\S{E}$ is a chain.
\end{lemma}

\begin{proof}
The successor map is order preserving with respect to the product order $\leq_{\S{U}}$ according to Lemma \ref{lemma:order-preserving}. The assertion follows, because any order is transitive. 
\end{proof}

The chain $\S{E} \subseteq \S{U}$ in Lemma \ref{lemma:warping-chain} is compatible with the successor map $S_{\S{U}}$. We call such a chain a warping chain.

\begin{definition}
Let $\S{U} = \S{V} \times \S{W}$ be the product of chains $\S{V}$ and $\S{W}$. A \emph{warping chain} in $\S{U}$ is a chain $\S{E} = \cbrace{e_1, \ldots, e_L} \subseteq \S{W}$ such that $e_{l+1} \in S_{\S{U}}(e_l)$ for all $l \in [L-1]$. 
\end{definition}

The next result shows that warping chains preserve the order of the factor chains. 

\begin{proposition}\label{prop:order-preserving-factors}
Let $\S{V}$ and $\S{W}$ be chains. Let $\S{E}$ be a warping chain in $\S{V} \times \S{W}$. Then any pair of elements $(i, j), (r, s) \in \S{E}$ satisfies
\begin{align}\label{eq:prop:order-preserving-factors}
\args{i \leq_{\S{V}} r \; \wedge \; j \leq_{\S{W}} s} \; \vee \; \args{r \leq_{\S{V}} i \; \wedge \; s \leq_{\S{W}} j}.
\end{align}
\end{proposition}

\begin{proof}
Suppose that $\S{E} = \cbrace{e_1, \ldots, e_L}$. Then there are indices $p, q \in [L]$ such that $e_p = (i,j)$ and $e_q = (r,s)$. Without loss of generality, we assume that $p \leq q$. Let $u = q - p$. Repeatedly applying Lemma \ref{lemma:order-preserving} yields
\[
e_p \leq_{\S{U}} \cdots \leq_{\S{U}} e_{p+u} = e_q,
\]
where $\S{U} = \S{V} \times \S{W}$. Since any order is transitive, we have $e_p \leq_{\S{U}} e_q$. Then the assertion directly follows from the definition of the product order $\leq_{\S{U}}$.
\end{proof}

Equation \eqref{eq:prop:order-preserving-factors} is the order-preserving property (or non-crossing property) of a warping chain. Note that the order preserving property does not hold for all subsets of $\S{V} \times \S{W}$.

\subsection{Warping Graphs}

This section introduces the notion of warping graph that models the alignment of two time series by a warping path and studies its local and global structure.

\medskip

A graph is a pair $G = \args{\S{V}, \S{E}}$ consisting of a finite set $\S{V} \neq \emptyset$ of nodes and a set $\S{E} \subseteq \S{V} \times \S{V}$ of edges. A node $i \in \S{V}$ is incident with an edge $e \in \S{E}$, if there is a node $j\in \S{V}$ such that $e = (i,j)$ or $e = (j,i)$. Similarly, an edge $(i,j) \in \S{E}$ is said to be incident to node $i$ and to node $j$. The neighborhood of node $i \in \S{V}$ is the subset of nodes defined by $\S{N}(i) = \cbrace{j \in \S{V} \,:\, (i, j) \in \S{E} \text{ or } (j,i) \in \S{E}}$. The elements of $\S{N}(i)$ are the neighbors of $i$. The degree $\deg(i) = \abs{\S{N}(i)}$ of node $i$ in $G$ is the number of neighbors of $i$. 

A subgraph of graph $G = \args{\S{V}, \S{E}}$ is a graph $G' = \args{\S{V}', \S{E}'}$ such that $\S{V}' \subseteq \S{V}$ and $\S{E}' \subseteq \S{E}$. We write $G' \subseteq G$ to denote that $G'$ is a subgraph of $G$. A graph $G$ is connected, if for any two nodes $i, j \in \S{V}$ there is a sequence $i = u_1, u_2, \ldots, u_n = j$ of nodes in $G$ such that $u_{k+1} \in \S{N}(u_k)$ for all $k \in [n-1]$. A component $C$ of graph $G$ is a connected subgraph $C \subseteq G$ such that $C \subseteq C'$ implies $C = C'$ for every connected subgraph $C' \subseteq G$.

A graph $G = \args{\S{U}, \S{E}}$ is bipartite, if $\S{U}$ can be partitioned into two disjoint and non-empty subsets $\S{V}$ and $\S{W}$ such that $\S{E} \subseteq \S{V} \times \S{W}$. We write $G = (\S{V}, \S{W}, \S{E})$ to denote a bipartite graph with node partitions $\S{V}$ and $\S{W}$. Note that the order of the node partitions $\S{V}$ and $\S{W}$ in a bipartite graph $G = \args{\S{U}, \S{E}}$ matters. A bipartite chain graph is a bipartite graph whose node partitions are chains.

\begin{definition}
A bipartite chain graph $G = \args{\S{V}, \S{W}, \S{E}}$ with node partitions $\S{V} = \cbrace{i_1, \ldots, i_m}$ and $\S{W} = \cbrace{j_1, \ldots, j_n}$ is a \emph{warping graph} of size $m \times n$ if 
\begin{enumerate}
\item $\args{i_1,j_1}, \args{i_m,j_n} \in \S{E}$ \hfill\emph{(\emph{boundary condition})}
\item $\S{E}$ is a warping chain in $\S{V} \times \S{W}$ \hfill\emph{(\emph{step condition})}, 
\end{enumerate}
\end{definition}
The set of all warping graphs of size $m \times n$ is denoted by $\S{G}_{m,n}$. If $G = \args{\S{V}, \S{W}, \S{E}}$ is a warping graph, we briefly write $S_G$ to denote the successor map $S_{\S{V} \times \S{W}}$ and $\leq_G$ to denote the induced product order $\leq_{\S{V}\times \S{W}}$. The following result is a direct consequence of the boundary and step conditions:

\begin{proposition}\label{prop:isolated-nodes}
Every node in a warping graph has a neighbor. 
\end{proposition}

We show that the neighborhood of a node of one partition of a warping graph is a contiguous chain of the other partition.

\begin{proposition}\label{prop:contiguous-neighborhood}
Let $G$ be a warping graph with node partitions $\S{Z}$ and $\S{Z}'$. Suppose that $i \in \S{Z}$ is a node. Then the neighborhood $\S{N}(i)$ of a node in $i \in \S{Z}$ is a contiguous subchain of $\S{Z}'$.
\end{proposition}

\begin{proof}
Suppose that the warping graph is of the form $G = \args{\S{V}, \S{W}, \S{E}}$. Without loss of generality, we assume that $\S{Z} = \S{V}$ and $\S{Z}' = \S{W}$. Then we have $i \in \S{V}$. The assertion trivially holds for $\abs{\S{N}(i)} = 1$. Suppose that $\abs{\S{N}(i)} > 1$. We assume that $\S{N}(i)$ is not contiguous. Then there are elements $j', j'' \in \S{N}(i)$ and $j \in \S{W} \setminus \S{N}(i)$ such that $j' \leq_G j \leq_G j''$. From Prop.~\ref{prop:isolated-nodes} follows that there is a node $i' \in \S{V}\setminus \cbrace{i}$ such that $(i',j) \in \S{E}$. 

Two cases can occur: (1) $i' <_{\S{V}} i$ and (2) $i <_{\S{V}} i'$. It is sufficient to consider the first case $i' <_{\S{V}} i$. The 
proof of the second case is analogue. By construction, there are edges $(i',j)$ and $(i,j')$ such that $i' <_{\S{V}} i$ and $j' <_{\S{W}} j$. These relationships violate the order preserving property of a warping chain given in Eq.~\eqref{eq:prop:order-preserving-factors} of Prop.~\ref{prop:order-preserving-factors}. Hence, $\S{N}(i)$ is contiguous.
\end{proof}

We introduce compact warping graphs that represent warping paths of minimal length.

\begin{definition}
A warping graph $G \in \S{G}_{m,n}$ is \emph{compact} if there is no warping graph $G' \in \S{G}_{m,n}$ such that $G'$ is a proper subgraph of $G$. 
\end{definition}

A warping graph is compact if no edge can be deleted without violating the boundary or step conditions. Figure \ref{fig:ex_compact_wg} shows an example of a non-compact warping graph and its compactification. 

\begin{figure}[t]
\centering
\includegraphics[width=0.6\textwidth]{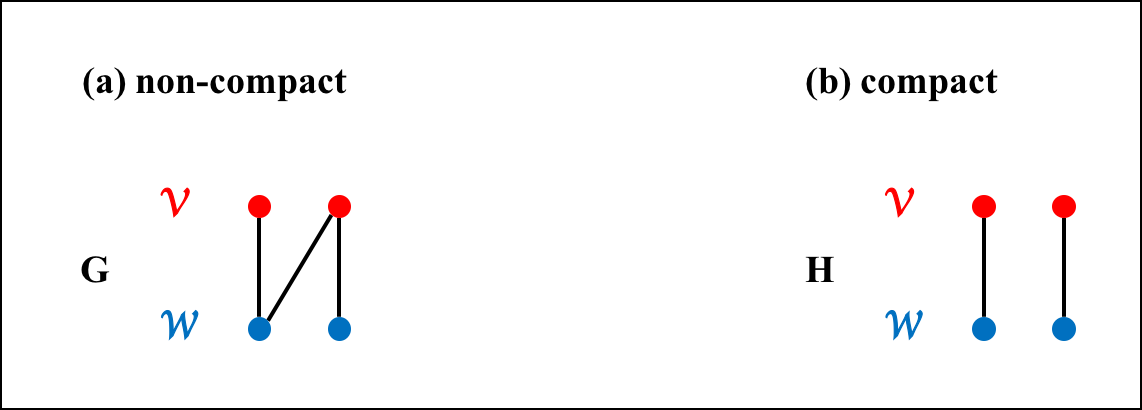}
\caption{Example of a non-compact warping graph $G$ (a) and its compactification $H$ (b).}
\label{fig:ex_compact_wg}
\end{figure}

\begin{proposition}\label{prop:characterization-compactness}
Let $G$ be a warping graph with edge set $\S{E} = \cbrace{e_1, \ldots, e_L}$. Then the following statements are equivalent:
\begin{enumerate}
\item $G$ is compact.
\item Let $2 \leq k < L$. Then $e_{l+k} \notin S_G(e_l)$ for all $l \in [L-k]$.
\end{enumerate}
\end{proposition}

\begin{proof}
We first prove the following Lemma:
\begin{lemma*}
Let $\S{V}$ and $\S{W}$ be two chains, let $\S{E}= \cbrace{e_1, \ldots, e_L} \subseteq \S{V} \times \S{W}$ be a warping chain, and let $3 \leq k < L$. Then $e_{l+k} \notin S_{\S{V} \times \S{W}}(e_l)$ for every $l \in [L-k]$. 
\end{lemma*}

\begin{proof}
Let $\S{C} = \cbrace{i_1, \ldots, i_m}$ be a chain and let $i_k, i_l \in \S{C}$ be two elements of $\S{C}$. Then we define the distance 
\[
\Delta_{\S{C}}\args{i_k,i_l} = \abs{l-k}+1.
\]
Suppose that $e_l = (i,j)$ for some $l \in [L-1]$. Let $e'=(r,s) \in S_{\S{V} \times \S{W}}(s_l)$ be an arbitrary successor of $e_l$. Then by definition of the successor map, we have $\Delta_{\S{V}}(i,r), \Delta_{\S{W}}(j,s) \in \cbrace{0,1}$. Let $l \in [L-k]$. Suppose that $e_l = (i,j)$ and $e_{l+k} = (r,s)$. Then by induction, we have 
\[
\Delta_{\S{V}}(i,r) + \Delta_{\S{W}}(j,s) \geq k.
\]
From $k \geq 3$ follows $\Delta_{\S{V}}(i,r) \geq 2$ or $\Delta_{\S{W}}(j,s) \geq 2$. Hence, $\Delta_{\S{V}}(i,r) \notin \cbrace{0,1}$ or $\Delta_{\S{W}}(j,s) \notin \cbrace{0,1}$. This shows the assertion $e_{l+k} \notin S_{\S{V} \times \S{W}}(e_l)$.
\end{proof}

From the above Lemma follows that the case $k > 2$ is impossible for a warping chain. Therefore it is sufficient to consider the case $k = 2$. 

Let $G$ be compact. We assume that there is an $l \in [L-2]$ such that $e_{l+2} \in S_G(e_l)$. By construction, the edge $e_{l+1}$ is an inner element of the chain $\S{E}$. From $e_{l+2} \in S_G(e_l)$ follows that removing $e_{l+1}$ neither violates the boundary conditions nor the step condition. This contradicts compactness of $G$ and shows that a compact warping graph $G$ implies the second statement. 

Next, we show the opposite direction. Suppose that $e_{l+2} \notin S_G(e_l)$ for all $l \in [L-2]$. We assume that $G$ is not compact. Then there is an edge $e_k \in \S{E}$ that can be removed without violating the boundary and step conditions. Not violating the boundary condition implies that $1 < k < L$. Hence, $e_{k-1}$ and $e_{k+1}$ are edges in $\S{E}$. We set $l = k-1$. Then we obtain the contradiction that $e_{l} \in S_G(e_{l+2})$. Hence, $G$ is compact.
\end{proof}

Suppose that $G = \args{\S{V}, \S{W}, \S{E}}$ is a warping graph. By $\S{V} \sqcup \S{W}$ we denote the disjoint union of the node partitions. If $i \in \S{V} \sqcup \S{W}$ is a node of one partition, then its neighborhood $\S{N}(i)$ is a subset of the other node partition. Hence, $\S{N}(i)$ is a chain and has boundary and eventually inner nodes. Let $\S{N}^\circ(i)$ denote the possibly empty subset of inner nodes of chain $\S{N}(i)$. We show that inner nodes of $\S{N}(i)$ always have degree one. 

\begin{lemma}\label{lemma:star-graph-decomposition:1}
Let $G = \args{\S{V}, \S{W}, \S{E}}$ be a warping graph. Suppose that $i \in \S{V} \sqcup \S{W}$ is a node with neighborhood $\S{N}(i)$. Then $\deg(j) = 1$ for all $j \in \S{N}^\circ(i)$. 
\end{lemma}

\begin{proof} Without loss of generality, we assume that $i \in \S{V}$. Then $\S{N}(i) \subseteq \S{W}$ is a chain. The assertion holds for $\abs{\S{N}(i)} \leq 2$, because in this case $\S{N}(i)$ has no inner node. Suppose that $\abs{\S{N}(i)} > 2$. Let $j \in \S{N}(i)$ be an inner node. We assume that $\deg(j) > 1$. Then there is a node $i' \in \S{V} \setminus \cbrace{i}$ such that $(i', j) \in \S{E}$. Since $\S{V}$ is a chain, we find that either $i' <_{\S{V}} i$ or $i <_{\S{V}} i'$. 

We only consider the first case $i' <_{\S{V}} i$. The proof of the second case is analogue. Observe that $j$ is an inner node of $\S{N}(i)$ and $\S{N}(i)$ is contiguous by Prop.~\ref{prop:contiguous-neighborhood}. Then $\S{N}(i)$ contains the predecessor $j' = j^-$ of node $j$. By construction $e = (i, j')$ and $e' = (i',j)$ are edges of $\S{E}$ such that $i' <_{\S{V}} i$ and $j' <_{\S{W}} j$. Thus, the edges $e$ and $e'$ violate the order preserving property of a warping chain given in Eq.~\eqref{eq:prop:order-preserving-factors} of Prop.~\ref{prop:order-preserving-factors}. This contradicts our assumption that $\S{E}$ is a warping chain. Hence, we have $\deg(j) = 1$. Since the inner node $j$ was chosen arbitrarily, the assertion follows. 
\end{proof}

\begin{lemma}\label{lemma:star-graph-decomposition:2}
Let $G = \args{\S{V}, \S{W}, \S{E}}$ be a warping graph and let $i \in \S{V} \sqcup \S{W}$ be a node with neighborhood $\S{N}(i)$. Suppose that $\abs{\S{N}(i)} \geq 2$ and $j \in \S{N}(i)$ is a boundary node with $deg(j) \geq 2$. Then the following properties hold:
\begin{enumerate}
\item 
If $j$ is the first node in $\S{N}(i)$, then $i^- \in \S{V}$ exists and $(i^-, j) \in \S{E}$.
\item 
If $j$ is the last node in $\S{N}(i)$, then $i^+ \in \S{V}$ exists and $(i^+, j) \in \S{E}$.
\end{enumerate}
\end{lemma}

\begin{proof} 
We show the second assertion. The proof of the first assertion is analogue. Since $\abs{\S{N}(i)} \geq 2$ and $j$ is the last node of $\S{N}(i)$, we find that $j' = j^- \in \S{N}(i)$ and therefore $(i,j') \in \S{E}$ exists. From $\deg(j) \geq 2$ follows that there is a node $i' \in \S{V}$ such that $(i',j) \in \S{E}$. We assume that $i' \in \S{N}(j)$ satisfies $i' <_{\S{V}} i$ . This implies that $e = (i,j')$ and $e' = (i',j)$ are two edges of $\S{E}$ such that $i' <_{\S{V}} i$ and $j' < _{\S{W}} j$. Thus, the edges $e$ and $e'$ violate the order preserving property of a warping chain given in Eq.~\eqref{eq:prop:order-preserving-factors} of Prop.~\ref{prop:order-preserving-factors}. This contradicts our assumption that $i' <_{\S{V}} i$. Therefore, we have $i <_{\S{V}} i'$. This in turn shows that $i^+ \in \S{V}$ exists. From $i, i' \in \S{N}(j)$ and $i <_{\S{V}} i'$ follows $i^+ \in \S{N}(j)$, because $\S{N}(j)$ is a contiguous subchain of $\S{V}$ by Prop.~\ref{prop:contiguous-neighborhood}. This shows $(i^+, j) \in \S{E}$ and completes the proof.
\end{proof}

A bipartite graph $G = (\S{V}, \S{W}, \S{E})$ is complete if $\S{E} = \S{V} \times \S{W}$. Let $r \in \N$. A complete bipartite graph $G = (\S{V}, \S{W}, \S{E})$ is a star graph of the form $K_{1,r}$, if $\abs{\S{V}} = 1$ and $\abs{\S{W}} = r$. Similarly, $G$ is a star graph of the form $K_{r,1}$, if $\abs{\S{V}} = r$ and $\abs{\S{W}} = 1$. By definition, a star graph has at least two nodes. A star forest is a graph whose components are star graphs.

\begin{proposition}\label{prop:star-graph-decomposition}
A compact warping graph is a star forest.
\end{proposition}

\begin{proof}
Let $G = \args{\S{V}, \S{W}, \S{E}}$ be a compact warping graph. Suppose that $C = \args{\S{V}', \S{W}', \S{E}'}$ is a component of $G$. From Prop.~\ref{prop:isolated-nodes} follows that $C$ has at least two nodes connected by an edge. 

We assume that $C$ is not a star. Then $C$ has two nodes $i, j \in \S{V}' \sqcup \S{W}'$ with degree larger than one. Without loss of generality, we assume that $i \in \S{V}'$. Then $\S{N}(i) \subseteq \S{W}'$ has at least two elements. Suppose that all nodes from $\S{N}(i)$ have degree one. Since component $C$ is bipartite, we find that $C$ is isomorphic to the star $K_{1,r}$, where $r = \deg(i) > 1$. This contradicts our assumption that $C$ is not a star. Hence, there is a node $j \in \S{N}(i) \subseteq \S{W}'$ with $\deg(j) > 1$.

From Lemma \ref{lemma:star-graph-decomposition:1} follows that node $j$ is a boundary node of $\S{N}(i)$. We show the assertion for the case that $j$ is the last node in $\S{N}(i)$. The proof for the case that $j$ is the first node in $\S{N}(i)$ is analogue. Since $j$ is the last node in $\S{N}(i)$ and $\abs{\S{N}(i)} \geq 2$, we have $j^- \in \S{N}(i)$ and therefore $(i,j^-) \in \S{E}$. Applying Lemma \ref{lemma:star-graph-decomposition:2} yields that $i^+ \in \S{V}$ exists and $(i^+, j) \in \S{E}$. 

By construction, we have $\args{i, j^-}, \args{i, j}, \args{i^+, j} \in \S{E}$. This shows that $(i,j)$ is not a boundary edge in $\S{E}$. Since $\args{i^+, j} \in S_{\S{V} \times \S{W}}\args{i, j^-}$, we can remove $(i,j)$ without violating the step condition. Then the subgraph $G' = \args{\S{V}, \S{W}, \S{E}\setminus \cbrace{(i,j)}}$ of $G$ is a warping graph. This contradicts our assumption that $G$ is compact. Hence, $C$ is a star. 
\end{proof}

An immediate consequence of the proof of Prop.~\ref{prop:star-graph-decomposition} is the following corollary. 
\begin{corollary}\label{cor:prop:star-graph-decomposition}
Let $G = \args{\S{V}, \S{W}, \S{E}}$ be a compact warping graph in $\S{G}_{m,n}$. Then every component of $G$ is a star with at least one node in $\S{V}$ and one node in $\S{W}$. 
\end{corollary}

\begin{proposition}\label{prop:2-star-graph-distribution}
Let $G \in \S{G}_{m,n}$ be a compact warping graph with $m > n$. Then we have:
\begin{enumerate}
\item $G$ has at most $n-1$ components of the form $K_{1,1}$. 
\item $G$ has a component of the form $K_{r,1}$ with $r > 1$. 
\end{enumerate}
\end{proposition}

\begin{proof}
Let $G = \args{\S{V}, \S{W}, \S{E}}$ with $\abs{\S{V}} = m$ and $\abs{\S{W}} = n$. To show the first assertion, we assume that $G$ has $n' > n-1$ components of the form $K_{1,1}$. This is only possible for $n' = n$, because every component of $G$ has at least one node in $\S{W}$ by Corollary \ref{cor:prop:star-graph-decomposition}. 

Let $C_1, \ldots, C_n$ be $n$ components of $G$ of the form $K_{1,1}$. Suppose that $\S{C}_k = \args{\cbrace{i_k}, \cbrace{j_k},  \cbrace{\args{i_k, j_k}}}$ for all $k \in [n]$. The union of the first node partitions over the $n$ components $C_k$ gives $\S{V}' = \cbrace{i_1, \ldots, i_n}$. From $m > n$ follows $\S{V}' \subsetneq \S{V}$. Then there is a node $i \in \S{V}\setminus\S{V}'$. From Prop.~\ref{prop:star-graph-decomposition} follows that there is a component $C$ of $G$ is a star of the form $K_{r,s}$ that contains node $i$. Since $s \geq 1$ by definition of a star, component $C$ has a node $j \in \S{W}$. Then there is a $k \in [n]$ such that $j = j_k$ is a node in component $\S{C}_k$. Since $i \neq i_k$ by construction, the graph $H = \args{\cbrace{i, i_k}, \cbrace{j_k}, \cbrace{\args{i,j_k}, \args{i_k,j_k}}}$ is a connected subgraph of $G$ that includes component $C_k$ as a proper subgraph. This contradicts our assumption that $C_k$ is a maximal connected subgraph of $G$. Consequently, $G$ cannot have more than $n-1$ components of the form $K_{1,1}$.

Next, we show the second assertion. Suppose that $C_1, \ldots, C_q$ are all components of $G$ that are of the form $K_{1,1}$. Let $\S{V}' = \cbrace{i_1, \ldots, i_q} \subseteq \S{V}$ and $\S{W}' = \cbrace{j_1, \ldots, j_q} \subseteq \S{W}$ be the subsets covered by the $q$ components $C_k$. From the first part of this proof follows that $q < n$ and by assumption, we have $n < m$. Then $\S{V}'' = \S{V}\setminus\S{V}'$ and $\S{W}'' = \S{W}\setminus\S{W}'$ are non-empty. By $m'' = \abs{\S{V}''} = m-q$ and $n'' = \abs{\S{W}''} = n-q$ we denote the respective number of nodes not contained in any of the $q$ components $C_k$. From $q < n < m$ follows that $1 \leq n'' < m''$. The pigeonhole principle states that there is at least one node $j \in \S{W}''$ that is connected to at least two nodes $i, i' \in \S{V}''$. Let $C$ be the component of $G$ containing the three nodes $i$, $i'$, and $j$. From Prop.~\ref{prop:star-graph-decomposition} follows that $C$ is a star of the form $K_{r,s}$. We find that $r \geq 2$, because $C$ contains at least the two nodes $i$ and $i'$ from $\S{V}'' \subset \S{V}$. Then $s = 1$, because $C$ is a star. This shows the second assertion.
\end{proof}

\begin{definition}
A node $i$ of a compact warping graph $G$ is \emph{redundant} if $\deg(j) \geq 2$ for all $j \in \S{N}(i)$.
\end{definition}

Let $i$ be a node in $G$. Then $G-\cbrace{i}$ is the subgraph of $G$ obtained by deleting node $i$ and its incident edges. The next result shows that deleting a redundant node of a compact warping graph is again a compact warping graph. 

\begin{proposition}\label{prop:reduce-redundant-node}
Let $i$ be a redundant node of a compact warping graph $G$. Then $G-\cbrace{i}$ is a compact warping graph. 
\end{proposition}
 
\begin{proof}
Without loss of generality, we assume that $i \in \S{V}$. Let $G = \args{\S{V}, \S{W}, \S{E}}$ with warping chain $\S{E} = \cbrace{e_1, \ldots, e_L}$. Suppose that $G' = G-\cbrace{i} = \args{\S{V}', \S{W}, \S{E}'}$, where $\S{V}' = \S{V} \setminus \cbrace{i}$ and $\S{E}'$ is the chain obtained from $\S{E}$ by removing all edges incident to node $i$. 

We first show that $\S{N}(i) \subseteq \S{W}$ consists of a singleton. Let $C$ be the component of $G$ that contains node $i$. Then $C$ also includes the neighborhood $\S{N}(i)$. Since $G$ is compact, we can apply Prop.~\ref{prop:star-graph-decomposition} and find that component $C$ is a star of the form $K_{r,1}$ or $K_{1,r}$, where $r \geq 1$. Observe that $r = \deg(j) \geq 2$ for every neighbor $j \in \S{N}(i)$, because node $i$ is redundant. Hence, $C$ is a star of the form $K_{r,1}$ and therefore $\S{N}(i) = \cbrace{j}$. 

From $\deg(j) \geq 2$ follows that there is a node $i' \in \S{V} \setminus \cbrace{i}$ such that $(i',j) \in \S{E}$. We distinguish between two cases: (1) $i' <_{\S{V}} i$ and (2) $i <_{\S{V}} i'$. We only consider the first case $i' <_{\S{V}} i$. The proof of the second case is analogue. From Prop.~\ref{prop:contiguous-neighborhood} follows that $\S{N}(j)$ is a contiguous subchain of $\S{W}$ with $i', i \in \S{N}(j)$. Then $i^- \in \S{N}(j)$. We distinguish between two cases: (1) $e_L = (i,j)$, and (2) $e_l = (i,j)$ for some $1 < l < L$. Note that the case $e_1 = (i,j)$ cannot occur due to existence of $i^- \in \S{V}$.

\medskip 

\noindent
\emph{Case 1:} From $e_L = (i,j)$ follows that the edge set of $G'$ is of the form $\S{E}' = \cbrace{e_2, \ldots, e_{L-1}}$. In addition, $i$ is the last node in $\S{V}$ and therefore $i^-$ is the last node in $\S{V}'$. Furthermore, we find that $j \in \S{W}$ is the last node in $\S{W}$. We show that $\S{E}'$ is a warping chain. The first boundary condition is satisfied by $e_1 \in \S{E}'$. From $i^- \in \S{N}(j)$ follows that $(i^-, j) \in \S{E}$ is an edge that satisfies the second boundary condition connecting the last nodes of $\S{V}'$ and $\S{W}$. Finally, from $e_{l+1} \in S_G(e_l)$ for all $l \in [L-1]$ follows that the step condition remains valid in $G'$. Therefore, the edge set $\S{E}'$ is a warping chain of length $L' = L-1$. It remains to show that $G'$ is compact. For this, we assume that $G'$ is not compact. Then from Prop.~\ref{prop:characterization-compactness} follows that there is an index $l \in [L'-2]$ such that $e_{l+2} \in S_{G'}(e_l)$. This implies that $e_{l+2} \in S_{G}(e_l)$ contradicting the assumption that $G$ is compact. Hence, $G'$ is a compact warping graph. 

\medskip 

\noindent
\emph{Case 2:} There is an index $1 < l < L$ such that $e_l = (i,j)$. Hence, $\S{E}$ has at least three edges and the edge set of $G'$ is of the form $\S{E}' = \cbrace{e_1, \ldots,e_{l-1}, e_{l+1}, \ldots, e_L}$. To show that $\S{E}'$ is a warping chain, we assume that $e_{l-1} = (i',j')$ and $e_{l+1} = (i'',j'')$. From $\S{N}(i) = \cbrace{j}$ and the step condition follows that $i' = i^-$ and $i'' = i^+$. This shows that $i$ is neither the first nor last node in $\S{V}$. Hence, $e_1$ and $e_L$ satisfy the boundary conditions in $\S{E}'$.

We show that $\S{E}'$ satisfies the step condition. Since $\S{E}$ is a warping chain, we have $e_{k+1} = S_G(e_k)$ for all $k \in [L-1]$.
Since $S_{G'} = S_G$ on $\S{E}'$, it is sufficient to show that $e_{l+1}\in S_{G'}(e_l)$. According to the previous parts of the proof, we have $(i^-,j) \in \S{E}$. The step condition together with $i' = i^-$ imply that $e_{l-1} = (i',j') = (i^-, j') = (i^-,j)$ and therefore $j' = j$. Again, from the step condition follows that either $j'' = j$ or $j'' = j^+$. Observe that $i'^{+} = i''$ in $\S{V}'$. Then we have $(i'', j) \in S_{G'}(i',j)$ and $(i'', j^+) \in S_{G'}(i',j)$. This shows that $\S{E}'$ satisfies the step condition in either of both cases $j'' = j$ or $j'' = j^+$. 

It remains to show that $G'$ is compact. For this, we assume that $G'$ is not compact. Suppose that $\S{I} = [L]\setminus \cbrace{l}$ is the index set of $\S{E}'$. Then from Prop.~\ref{prop:characterization-compactness} follows that there is an index $k \in \S{I}\setminus\cbrace{L-1,L}$ such that $e_{k+2} \in S_{G'}(e_k)$. This implies that $e_{k+2} \in S_{G}(e_k)$ contradicting the assumption that $G$ is compact. Hence, $G'$ is a compact warping graph. 
\end{proof}

\subsection{Glued Warping Graphs}\label{subsec:glued-graphs}

This section glues warping graphs to model Fr\'echet function. Then the graph-theoretic foundation of the Reduction Theorem is presented.

\medskip

A graph $G = \args{\S{U}, \S{E}}$ is a centered $N$-partite graph if the set $\S{U}$ can be partitioned into $N+1$ disjoint and non-empty subsets $\S{V}, \S{W}_1 \ldots, \S{W}_N$ such that 
\[
\S{E} \subseteq \bigcup_{k=1}^N \S{V} \times \S{W}_k.
\]

\begin{definition}
Let $G_1, \ldots, G_N$ be compact warping graphs of the form $G_k = \args{\S{V}, \S{W}_k, \S{E}_k}$ for all $k \in [N]$. The \emph{glued graph} with \emph{splice} $\S{V}$ and \emph{particles} $G_1, \ldots, G_N$ is a centered $N$-partite graph $G = \args{\S{V}, \S{W}_1, \ldots, \S{W}_N, \S{E}}$ with edge set $\S{E} = \S{E}_1 \sqcup \cdots \sqcup \S{E}_N$. 
\end{definition}

Note that a particle of a glued graph is always a compact warping graph. The definition of a glued graph assumes that all $N$ particles $G_1, \ldots, G_N$ share a common node partition $\S{V}$ and that any two particles $G_k$ and $G_l$ have disjoint node partitions $\S{W}_k$ and $\S{W}_l$, respectively. Then the glued graph with splice $\S{V}$ is obtained by taking the disjoint union of the particles $G_1, \ldots, G_N$, but by identifying the nodes from $\S{V}$. A special case of a glued graph is any compact warping graph $G = \args{\S{V}, \S{W}, \S{E}}$ with the first partition $\S{V}$ as its splice. 

\commentout{
Suppose that $G$ is a glued graph with splice $\S{V}$. By construction, $\S{V}$ is a chain. The predecessor of the first splice node and the successor of the last splice node is the void node, denoted by $*$. To avoid case distinctions, we define the neighborhood of the void node as the empty set, that is $\S{N}(*) = \emptyset$. 
}

\begin{definition}
Let $G$ be a glued graph with splice $\S{V}$ and particles $G_1, \ldots, G_N$. Node $i \in \S{V}$ is \emph{redundant} in $G$, if it is redundant in $G_k$ for every $k \in [N]$. 
\end{definition}

\begin{proposition}\label{prop:reduced-glued-graph}
Let $G$ be a glued graph with splice $\S{V}$ and particles $G_1, \ldots, G_N$. Suppose that $i \in \S{V}$ is redundant. Then $G-\cbrace{i}$ is a glued graph with splice $\S{V}\setminus\cbrace{i}$ and particles $G_1-\cbrace{i}, \ldots, G_N-\cbrace{i}$.
\end{proposition}
 
\begin{proof}
Let $k \in [N]$. Then particle $G_k = \args{\S{V}, \S{W}_k, \S{E}_k}$ is a compact warping graph by definition of a glued graph. Since splice node $i \in \S{V}$ is redundant in $G$, it is redundant in $G_k$. Then from Prop.~\ref{prop:reduce-redundant-node} follows that $G_k' = G_k - \cbrace{i} = \args{\S{V}', \S{W}_k, \S{E}'}$ is a compact warping graph with $\S{V}' = \S{V}\setminus\cbrace{i}$ and $\S{E}_k' = \S{E}_k\setminus \S{E}_k(i)$, where $\S{E}_k(i) \subseteq \S{E}_k$ is the subset of edges in $G_k$ incident to node $i$.

The graph $G' = G-\cbrace{i} = \args{\S{V}', \S{W}_1, \ldots, \S{W}_N, \S{E}'}$ has an edge set of the form $\S{E}' = \S{E}\setminus\S{E}(i)$, where $\S{E}(i) \subseteq \S{E}$ is the subset of edges in $G$ incident to node $i$. Since $\S{E} = \S{E}_1 \sqcup \cdots \sqcup \S{E}_N$, we have 
\[
\S{E}(i) = \S{E}_1(i) \sqcup \cdots \sqcup \S{E}_N(i) = \S{E}_1'\sqcup \cdots \sqcup \S{E}_N'.
\]
This shows that $G-\cbrace{i}$ is a glued graph of particles $G_1', \ldots, G_N'$ along splice $\S{V}'$.
\end{proof}

Suppose that $G$ is a glued graph with splice $\S{V}$ and particles $G_1, \ldots, G_N$. A particle is said to be \emph{trivial} if it is a star of the form $K_{m,1}$. By $\S{I}_G \subseteq [N]$ we denote the subset of all indices $k \in [N]$ for which particle $G_k$ is non-trivial. We call $\S{I}_G$ the \emph{core index set} (core) of $G$.

\begin{definition}
Let $G$ be a glued graph with splice $\S{V}$ and $N$ particles $G_k = \args{\S{V}, \S{W}_k, \S{E}_k}$ for all $k \in [N]$. Suppose that $\S{I}_G$ is the core index set of $G$. Then
\[
\rho(G) = \begin{cases}
\displaystyle \sum_{k \in \S{I}_G} \abs{\S{W}_k} - \;2\args{\abs{\S{I}_G}-1} & \S{I}_G \neq \emptyset \\
1 & \S{I}_G = \emptyset
\end{cases}
\]
is the \emph{reduction bound} of $G$. 
\end{definition}

Now we present the graph-theoretic foundation of the Reduction Theorem.
\begin{theorem}\label{prop:existence-redundant-node}
Let $G$ be a glued graph with splice $\S{V}$ such that $\rho(G) < \abs{\S{V}}$. Then $\S{V}$ has a redundant node.
\end{theorem}

\begin{proof}
Suppose that $G_1, \ldots, G_N$ are the particles of $G$ with $G_k = \args{\S{V}, \S{W}_k, \S{E}_k}$ for all $k \in [N]$. Let
$m = \abs{\S{V}}$, $n_k = \abs{\S{W}_k}$, and $\S{N}_k(i) = \S{N}(i) \cap \S{W}_k$ for every $k \in [N]$. 

\medskip 

We first consider the special case that $\S{I}_G = \emptyset$. Then all $N$ particles are trivial, that is $n_k = 1$ for every $k \in [N]$. The reduction bound is of the form $\rho(G) = 1$. By assumption, we have $m > \rho(G)$. Hence, every particle $G_k$ is a star of the form $K_{m,1}$, where $m > 1$. This shows that every splice node  is redundant. 

\medskip 

Next, we assume that $\S{I}_G \neq \emptyset$. We set $N' = \abs{\S{I}_G}$. Obviously, we have $N' \geq 1$. We say, $G_k$ supports node $i \in \S{V}$, if there is a node $j \in \S{N}_k(i) \subseteq \S{W}_k$ with $\deg(j) = 1$. It is sufficient to show that $\S{V}$ has a node not supported by any of the $N$ particles $G_1, \ldots, G_N$. The proof proceeds in four steps.

\paragraph*{\textmd{1.}} We show that $n_k < m$ for any $k \in \S{I}_G$. Suppose that $\S{J} = \S{I}_G\setminus\cbrace{k}$. Then from $\S{I}_G \neq \emptyset$ follows
\begin{align*}
\rho(G) = \sum_{l \in \S{I}_G} n_l - 2(N'-1) = n_k + \sum_{l \in \S{J}} n_l - 2(N'-1). 
\end{align*}
From $l \in \S{I}_G$ follows $n_l \geq 2$. This together with $\abs{\S{J}} = N'-1$ yields
\[
\rho(G) 
\geq n_k + \sum_{l \in \S{J}} 2 -2(N'-1)
= n_k + 2\abs{\S{J}} -2(N'-1)
= n_k.
\]
Then from $m > \rho(G)$ follows $m > n_k \geq 2$. 

\paragraph*{\textmd{2.}} We bound the number of splice nodes that can be supported by any non-trivial particle. For any $k \in \S{I}_G$ let $\S{W}_k'\subseteq \S{W}_k$ be the subset of nodes in $G_k$ that support a splice node. We define a map $\phi_k:\S{W}_k' \rightarrow \S{V}$ such that $(\phi_k(j), j) \in \S{E}_k$. Such a map exists due to the boundary and step conditions of warping graph $G_k$. Moreover, the map $\phi_k$ is uniquely determined, because $\deg(j) = 1$ for any node $j \in \S{W}_k'$. This shows that $\S{V}_k = \phi_k\!\args{\S{W}_k'}$ is the set of splice nodes supported by $G_k$. Since $\phi_k$ is bijective, we have $\abs{\S{V}_k} = \abs{\S{W}_k'}$. 

From step 1 of this proof follows that $G_k$ is a compact warping graph with $n_k < m$. Therefore, we can apply Prop.~\ref{prop:2-star-graph-distribution} and obtain that $G_k$ has at most $n_k-1$ components of the form $K_{1,1}$ and at least once component of the form $K_{r,1}$ with $r > 1$. This shows that $\abs{\S{W}_k'} = \abs{\S{V}_k} \leq n_k-1$.

\paragraph*{\textmd{3.}} We show that there is a splice node not supported by any non-trivial particle. For this, we define the set
\[
\S{U} = \bigcup_{k\in \S{I}_G} \S{V}_k
\]
of all splice nodes that are supported by at least one non-trivial particle of $G$. Then it is sufficient to show that $m > \abs{\S{U}}$. We consider three cases: (1) $N' = 1$, (2) $N' = 2$, and (3) $N' > 2$.

\medskip 

\noindent
\emph{Case 1: $N' = 1$.} Suppose that $\S{I}_G = \cbrace{u}$. Since $\S{I}_G \neq \emptyset$, the reduction bound is of the form
\[
\rho(G) = n_u - 2(N'-1) = n_u \geq 2.
\]
According to step 2, we have $n_u-1 \geq \abs{\S{V}_u}$. By using $\S{U} = \S{V}_u$ we find that
\[
m > \rho(G) > \abs{\S{V}_u} = \abs{\S{U}}.
\]

\medskip 

\noindent
\emph{Case 2: $N' = 2$.} Suppose that $\S{I}_G = \cbrace{u, v}$. Since $\S{I}_G \neq \emptyset$, the reduction bound takes the form
\[
\rho(G) = n_u + n_v - 2(N'-1) = n_u + n_v - 2 = (n_u -1)+(n_v-1).
\]
According to step 2, we have $n_u-1 \geq \abs{\S{V}_u}$ and $n_v-1 \geq \abs{\S{V}_v}$. By using $\S{U} = \S{V}_u \cup \S{V}_v$ we obtain  
\[
m > \rho(G) \geq \abs{\S{V}_u} + \abs{\S{V}_v} \geq \abs{\S{U}}.
\]

\medskip 

\noindent
\emph{Case 3: $N' > 2$.}
Suppose that the slice $\S{V}$ is a chain of the form $\S{V} = \cbrace{i_1, \ldots, i_m}$ with boundary nodes $\bd(\S{V}) = \cbrace{i_1, i_m}$. We assume that $\abs{\S{U}} = m$. Then there are (not necessarily distinct) indices $u,v \in \S{I}_G$ such that $i_1 \in \S{V}_u$ and $i_m \in \S{V}_v$. From the boundary and step conditions follows that the boundary nodes of any $W_k$ ($k \in \S{I}_G$) can only support the respective boundary nodes of $\S{V}$ and not any other splice node. Then the first node of $\S{W}_u$ only supports $i_1 \in \S{V}$ and the last node of $\S{W}_v$ only supports $i_m \in \S{V}$. 

Let $\S{J} = \cbrace{u,v}$, let $\S{J}' = \S{I}_G \setminus \S{J}$, and let $\S{V}_k' = \S{V}_k \setminus \args{\S{V}_u \cup \S{V}_v}$ for all $k \in \S{J}'$. The set $\S{V}_k'$ consists of all splice nodes supported by $G_k$ but not by $G_u$ and $G_v$. Hence, the boundary nodes of $\S{V}$ are not contained in $\S{V}_k'$. Then both boundary nodes of $\S{W}_k$ do not support any node in $\S{V}_k'$. This implies $\abs{\S{V}_k'} \leq n_k-2$. From the cardinality of the set union follows
\begin{align*}
\abs{\S{U}} 
&\leq \sum_{l \in \S{J}} \args{n_l - 1} +  \sum_{k\in \S{J}'} \args{n_k-2} 
= \abs{\S{J}} + \sum_{l \in \S{J}} \args{n_l - 2} +  \sum_{k\in \S{J}'} \args{n_k-2}
= \sum_{k \in \S{I}_G} n_k - 2N' + \abs{\S{J}}.
\end{align*}
Since $\abs{\S{J}} \leq 2$, we obtain
\[
\abs{\S{U}} \leq \sum_{k \in \S{I}_G} n_k - 2(N'-2) = \rho(G) < m.
\]
This contradicts the assumption $\abs{\S{U}} = m$ and shows that $\abs{\S{U}} < m$ holds. 

\medskip

All three cases show that $\abs{\S{U}} < m$. Hence, $G$ has a splice node not supported by any of the non-trivial particles.

\paragraph*{\textmd{4.}} We show that $G$ has a splice node not supported by any of the trivial and non-trivial particles. The non-trivial part follows from step 3. Therefore, it is sufficient to consider trivial particles only. Since $\S{I}_G \neq \emptyset$ by assumption, there is a $k \in \S{I}_G$. From $m > n_k$ and $n_k \geq 2$ follows $m > 2$. This implies that the trivial particles of $G$ do not support any of the splice nodes in $\S{V}$. This completes the proof.
\end{proof}

\subsection{Labeled Warping Graphs}

This section labels the nodes of warping graphs with the attributes of the corresponding time series to be aligned. Then we introduce the weight of a labeled warping graph for modeling the cost of aligning two time series along a warping path. 

\medskip

We assume that $\S{A}$ is an attribute set and $d: \S{A} \times \S{A} \rightarrow \R$ is a non-negative distance function on $\S{A}$. 

\begin{definition}
A \emph{labeled warping graph} $H = \args{G, \lambda}$ consists of a warping graph $G = \args{\S{V}, \S{W}, \S{E}}$ and a labeling function $\lambda: \S{V} \sqcup \S{W} \rightarrow \S{A}$. 
\end{definition}
The labeling function $\lambda$ assigns an attribute $\lambda(i) \in \S{A}$ to any node $i \in \S{V} \sqcup \S{W}$. Thus, the nodes correspond to time points and the attributes to the elements at every time point. 

The set of all labeled warping graphs of order $m \times n$ with label function $\lambda$ is denoted by $\S{G}_{m,n}^\lambda$. Since the set $\S{G}_{m,n}^\lambda$ fixes both node partitions and the label function, the graphs in $\S{G}_{m,n}^\lambda$ differ only in their edge sets. Thus, $\S{G}_{m,n}^\lambda$ describes the set of all possible warping paths that align time series $x = (x_1, \ldots, x_m)$ and $y = (y_1, \ldots, y_n)$ whose elements $x_i = \lambda(i)$ and $y_j = \lambda(j)$ are specified by the labeling function $\lambda$.

\begin{definition}
Let $H = \args{G, \lambda}$ be a labeled warping graph with edge set $\S{E}$. The \emph{weight} of $H$ is defined by
\[
\omega\args{H} = \sum_{(i,j) \in \S{E}} d(\lambda(i),\lambda(j))
\]
\end{definition}
The weight of a labeled warping graph corresponds to the cost of aligning two time series along a warping path. A DTW-graph is a labeled warping graph with minimal weight. 
\begin{definition}
A graph $H \in \S{G}_{m,n}^\lambda$ is a \emph{DTW-graph}, if 
\[
\omega(H) = \min \cbrace{\omega(H') \,:\, H' \in \S{G}_{m,n}^\lambda}.
\] 
\end{definition}
The weight of a DTW-graph is the DTW-distance between the time series represented by the labeled node partitions.

\subsection{Proofs of Results from Section \ref{sec:reduction-theorem+implications}}\label{subsec:proofs}

\paragraph*{Proof of Theorem \ref{theorem:reduction}.}

\begin{theorem*}[Reduction Theorem]
Let $F$ be the Fr\'echet function of a sample $\S{X}\in \S{T}^N$. Then for every time series $x \in \S{T}$ of length $\ell(x) > \rho(\S{X})$ there is a time series $x' \in \S{T}$ of length $\ell(x') = \ell(x) -1$ such that $F(x') \leq F(x)$. 
\end{theorem*}

\begin{proof}
Let $\S{X} = \args{x^{(1)}, \ldots, x^{(k)}}$, $m = \ell(x)$, and $n_k = \ell\args{x^{(k)}}$ for all $k \in [N]$. By assumption, we have $m > \rho(\S{X})$. 

For every $k \in [N]$ there is an optimal warping path $p^{(k)} \in \S{P}_{m, n_k}$ aligning $x$ and $x^{(k)}$. Let $H_k = (G_k, \lambda_k)$ be the DTW-graph representing $p^{(k)}$. Then $\omega(H_k) = \dtw(x,x^{(k)})$ and $G_k =\args{\S{V}, \S{W}_k, \S{E}_k} \in \S{G}_{m,n_k}$ is a warping graph with $m = \abs{\S{V}}$ and $n_k = \abs{\S{W}_k}$. Then we have 
\[
F(x) = \frac{1}{N} \sum_{k=1}^N h_k(\omega(H_k)),
\]
where $h_1, \ldots, h_N$ are the corresponding loss functions. 

Suppose that $G_k$ is non-compact and $G'_k \subseteq G_k$ is compact. Since $H_k$ is a DTW-graph, we have $\omega(H_k') = \omega(H_k)$, where  $H_k' = (G_k', \lambda_k')$. Hence, without loss of generality we can assume that $G_k$ is compact for all $k \in [N]$. Let $G$ be the glued graph with splice $\S{V}$ and particles $G_1, \ldots, G_N$. Since $m > \rho(G)$, we can apply Prop.~\ref{prop:existence-redundant-node} and obtain that $G$ has a redundant splice node $i \in \S{V}$. Applying Prop.~\ref{prop:reduced-glued-graph} yields that $G' = G-\cbrace{i}$ is a glued graph with splice $\S{V}' = \S{V}\setminus\cbrace{i}$ and particles $G_1', \ldots, G_N'$. The particles $G_k'$ are of the form $G_k' = G_k-\cbrace{i} = \args{\S{V}', \S{W}_k, \S{E}_k'}$, where the edge set $\S{E}_k'$ is obtained from $\S{E}_k$ by removing all edges incident to splice node $i \in \S{V}$. 

The redundant node $i\in \S{V}$ refers to element $x_i$ of time series $x = (x_1,\ldots, x_m)$. By 
\[
x' = (x_1, \ldots, x_{i-1}, x_{i+1}, \ldots, x_m)
\] 
we denote the time series obtained from $x$ by removing element $x_i$. Let $H_k' = (G_k', \lambda_k')$ be the resulting labeled warping graph, where $\lambda_k'$ denotes the labeling function obtained by restricting $\lambda_k$ to the subset $\S{V}' \sqcup \S{W}_k$ for all $k \in [N]$. Then the labeled warping graphs $H_k'$ represent warping paths $q^{(k)}$ that align time series $x'$ with sample time series $x^{(k)}$. By construction and definition of the weight function $\omega$, we find that $\omega(H_k') \leq \omega(H_k)$. Since the loss functions $h_k$ are monotonously increasing, we obtain 
\[
F(x') = \frac{1}{N} \sum_{k=1}^N h_k\args{\omega(H_k')} \leq \frac{1}{N} \sum_{k=1}^N h_k\args{\omega(H_k)} = F(x).
\]
By construction, we have $\ell(x') = \ell(x)-1$. This completes the proof. 
\end{proof}

\commentout{
One implication of Theorem \ref{theorem:reduction} is that existence of a sample mean entails existence of a sample mean, whose length is bounded by the reduction bound. 
\begin{corollary}\label{cor:existence:short-mean}
Let $\S{X}\in \S{T}^N$ be a sample and let $\rho\in \N$ be the reduction bound of $\S{X}$. Suppose that $\S{X}$ has a sample mean. Then there is a sample mean $z \in \S{F}$ with $\ell(z) \leq \rho$.
\end{corollary}

\begin{proof}
Suppose that $z$ is a sample mean of $\S{X}$ of minimum length $m = \ell(z)$. We assume that $m > \rho$. From Theorem \ref{theorem:reduction} follows that there is a time series $z' \in \S{T}$ of length $\ell(z') = m-1$ such that $F(z') \leq F(z)$. Since $z$ is a sample mean, we have $F(z) = F(z')$. Hence, $z'$ is also a sample mean of length $m-1$. This contradicts our assumption that $z$ is a sample mean of minimum length. Consequently, there is a sample mean $z'\in \S{F}$ of length $\ell(z') \leq \rho$.
\end{proof}
}

\paragraph*{Proof of Corollary \ref{cor:existence:restricted-to-unrestricted}.}

\begin{corollary*}
Let $\S{X}\in \S{T}^N$ be a sample and let $\rho \in \N$ be the reduction bound of $\S{X}$. Suppose that $\S{F}_m \neq \emptyset$ for every $m \in \bracket{\rho}$. Then $\S{X}$ has a sample mean.
\end{corollary*}

\begin{proof}
For every $m\in \bracket{\rho}$ let $F_m^*$ denote the restricted sample variance. We assume that $\S{F} = \emptyset$. Then there is a time series $x \in \S{T}$ of length $\ell(x) = p$ such that $F(x) = F_p(x) < F_m^*$ for all $m \in \bracket{\rho}$. This implies $p > \rho$, because otherwise we obtain the contradiction that $F_p(x) < F_p^*$. Let $q = p - \rho(\S{X})$. By applying Theorem \ref{theorem:reduction} exactly $q$-times, we obtain a time series $x' \in \S{T}$ of length $\ell(x') = \rho$ such that $F_{\rho}^* \leq F(x') \leq F(x)$. This contradicts our assumption that $F(x) < F_{\rho}^*$. Hence, $\S{F}$ is non-empty. 
\end{proof}

\paragraph*{Proof of Proposition \ref{prop:restricted-existence-discrete-case}.}

\begin{proposition*}
Let $\S{X} \in \S{T}^N$ be a sample. Suppose that $\S{A}$ is a finite attribute set. Then the following statements hold:
\begin{enumerate}
\itemsep0em
\item $\S{F}_m \neq \emptyset$ for every $m\in \N$. 
\item $\S{F}\neq \emptyset$.
\end{enumerate}
\end{proposition*}

\begin{proof}
Let $m \in \N$ be arbitrary. Since $\S{A}$ is finite, the set subset $\S{T}_m$ is also finite and consists of $m^{\abs{\S{A}}}$ time series. Then the set $F\args{\S{T}^N}$ is a finite set. Hence, the restricted sample mean set $\S{F}_m$ is non-empty and finite. Since $m$ was chosen arbitrarily, the first assertion follows. The second assertion follows from Corollary \ref{cor:existence:restricted-to-unrestricted}.
\end{proof}

\paragraph*{Proof of Proposition \ref{prop:sufficient-conditions-of-existence}.} 

\begin{proposition*}
Let $\S{X} \in \S{T}^N$ be a sample. Suppose that the following assumptions hold:
\begin{enumerate}
\item $\args{\S{A}, d}$ is a metric space of the form $\args{\R^d, \norm{\cdot}}$, where $\norm{\cdot}$ is a norm on $\R^d$. 
\item The loss functions $h_1, \ldots, h_N$ are continuous and strictly monotonously increasing.
\end{enumerate}
Then the following statements hold:
\begin{enumerate}
\itemsep0em
\item $\S{F}_m \neq \emptyset$ for every $m\in \N$. 
\item $\S{F}\neq \emptyset$.
\end{enumerate}
\end{proposition*}

The proof of Prop.~\ref{prop:sufficient-conditions-of-existence} uses the notion of coercive function. A  continuous function $f:\R^q \rightarrow \R$ is coercive if 
\[
\lim_{\norm{x} \to \infty} f(x) = + \infty,
\]
where $\norm{\cdot}$ is a norm on $\R^q$. 

\begin{proof}
We first consider the Euclidean norm $\norm{\cdot}_{2}$ on some real-valued vector space $\R^q$. The Euclidean norm is coercive. Since $h_k$ is continuous and strictly monotonously increasing on $\R_{\geq 0}$, the composition $h_k\args{\norm{x}_2}$ is coercive and continuous for all $k \in [N]$. Every norm $\norm{\cdot}$ on $\R^q$ is equivalent to the Euclidean norm. Therefore, we can find constants $0 < c \leq C$ such that 
\[
c\normS{x}{_2} \leq \normS{x} \leq C\normS{x}{_2}
\]
for all $x \in \R^q$. Hence, $h_k\args{\norm{x}}$ is coercive and continuous for every norm on $\R^q$.

Suppose that $\S{X} = \args{x^{(1)},\dots,x^{(N)}} \in \S{T}^N$ is a proper sample of $N$ time series $x^{(k)}$ of length $\ell\args{x^{(k)}} = n_k \geq 2$ for all $k \in [N]$. Let $m \in \N$ be arbitrary. Expanding the definition of the restricted Fr\'echet function $F_m$ gives
\begin{align*}
F_m(x) = \frac{1}{N}\sum_{k=1}^N h_k\args{\dtw\args{x, x^{(k)}}} = \frac{1}{N}\sum_{k=1}^N \min \cbrace{h_k\args{c_p\args{x, x^{(k)}}} \,:\, p \in \S{P}},
\end{align*}
where $c_p(x, y)$ is the cost of aligning time series $x$ and $y$ along warping path $p$. Since $\S{T}_m = \S{A}^m = \R^{d \times m} = \R^q$, we can define the function 
\[
g_{p^{(k)}}:\R^q \rightarrow \R, \quad x \mapsto c_{p^{(k)}}\args{x, x^{(k)}} = \sum_{l=1}^{L_k} h_k\args{\norm{x_{i_l} - x_{j_l}^{(k)}}},
\]
where $p^{(k)} \in \S{P}_{m, n_k}$ is a warping path with $L_k$ elements aligning $x$ and $x^{(k)}$. The function $g_{p^{(k)}}$ is continuous and coercive as a sum of non-negative continuous and coercive functions. Then $g_{p^{(k)}}$ has a global minimum.

We define the set $\S{P}_m = \S{P}_{m,n_1} \times \cdots \times \S{P}_{m,n_N}$. Then every element of $\S{P}_m$ is of the form $\S{C} = \args{p^{(1)}, \ldots, p^{(N)}}$, where $p^{(k)}$ is associated to time series $x^{(k)}$ for all $k \in [N]$. Then we can equivalently rewrite the restricted Fr\'echet function $F_m(x)$ as
\[
F_m(x) = \min \cbrace{F_{\S{C}}(x) \,:\, \S{C} \in \S{P}_m}, 
\]
where the component functions $F_{\S{C}}: \R^{d \times m} \rightarrow \R$ are functions of the form
\[
F_{\S{C}}(x) = \frac{1}{N}\sum_{k=1}^N g_{p^{(k)}}(x).
\]
This shows that $F_{\S{C}}(x)$ has a minimum. Let $F_{\S{C}}^*$ denote the minimum value of $F_{\S{C}}(x)$. From
\[
\min_x F_m(x) = \min_x \min_{\S{C}} F_{\S{C}}(x) = \min_{\S{C}} \min_x F_{\S{C}}(x)
\]
follows
\[
\min_x F_m(x) = \min_{\S{C} \in \S{P}_m} F_{\S{C}}^*.
\]
Since $\S{P}_m$ is a finite set, we obtain that $F_m$ has a minimum. This shows the first assertion, because $m$ was arbitrary. The second assertion follows from Corollary \ref{cor:existence:restricted-to-unrestricted}. 
\end{proof}

\paragraph*{Proof of Remark \ref{cor:existence-of-weighted-mean}.}

\begin{remark*}
Proposition \ref{prop:sufficient-conditions-of-existence} holds when we replace the loss functions $h_k$ by the loss functions $h'_k = w_k h_k$ with $w_k \in \R_{\geq 0}$ for all $k \in [N]$. \end{remark*}

\begin{proof}
Suppose that all weights are zero. Then the Fr\'echet function corresponding to the loss functions $h'_k = 0$ is zero. Hence, every time series $z \in \S{T}$ is an optimal solution and the assertion follows. 

We assume that at least one weight is non-zero. Without loss of generality, let $r \in [N]$ such that  $w_1, \ldots, w_r > 0 $ and $w_{r+1} = \cdots w_N = 0$. Then the loss functions $h'_1, \ldots, h'_r$ are continuous and strictly monotonously increasing. Moreover, the Fr\'echet function $F(z)$ of sample $\S{X}$ corresponding to the loss functions $h'_1, \ldots h'_N$ coincides with the Fr\'echet function $F'(z)$ of sample $x^{(1)}, \ldots, x^{(r)}$ corresponding to the loss functions $h'_1, \ldots h'_r$.  Then the assertion follows from Prop.~\ref{prop:sufficient-conditions-of-existence}.
\end{proof}

\end{appendix}

\end{document}